\documentclass[10pt,twocolumn]{article}

\usepackage{times}                          
\usepackage[T1]{fontenc}
\usepackage[margin=2.0cm,top=2.5cm,bottom=2.5cm]{geometry}
\usepackage{amsmath,amssymb,amsthm}
\usepackage{booktabs}
\usepackage{xcolor}
\usepackage{graphicx}
\usepackage{pgfplots}
\pgfplotsset{compat=1.18}
\usepackage[ruled,linesnumbered]{algorithm2e}
\usepackage{hyperref}
\usepackage[numbers,sort&compress]{natbib}
\usepackage{url}
\usepackage[expansion=false]{microtype}
\usepackage{enumitem}
\usepackage{tabularx}
\usepackage{multirow}

\newtheorem{proposition}{Proposition}

\SetAlgoLined
\SetKwProg{Fn}{Function}{:}{end}
\SetKwFunction{Train}{train}
\SetKwFunction{Forward}{HRM\_LM.forward}
\SetKwFunction{OnePass}{\_one\_pass}
\SetKwFunction{OutputHead}{\_output\_head}
\SetKwFunction{SG}{sg}

\definecolor{darkblue}{HTML}{1a3a5c}
\definecolor{medblue}{HTML}{2c5f8a}

\hypersetup{colorlinks=true,linkcolor=darkblue,citecolor=darkblue,urlcolor=medblue,
  pdftitle={Hierarchical vs. Flat Iteration in Shared-Weight Transformers},
  pdfauthor={Sang-Il Han},
  pdfkeywords={language modeling, hierarchical recurrent model, two-speed architecture, parameter sharing, KV cache compression}}

\usepackage{titlesec}
\titleformat{\section}{\normalsize\bfseries}{\thesection.}{0.5em}{}
\titleformat{\subsection}{\small\bfseries}{\thesubsection.}{0.4em}{}
\titlespacing{\section}{0pt}{8pt}{4pt}
\titlespacing{\subsection}{0pt}{5pt}{2pt}

\renewenvironment{abstract}{%
  \centerline{\normalsize\bf Abstract}%
  \vspace{0.5ex}%
  \begin{quote}\small%
}{\end{quote}\vspace{1ex}}

\DeclareMathOperator{\sg}{sg}
\DeclareMathOperator{\RMSNorm}{RMSNorm}
\DeclareMathOperator{\SwiGLU}{SwiGLU}

\begin{document}

\twocolumn[{%
\begin{center}
  {\LARGE\bfseries
   Hierarchical vs.\ Flat Iteration\\[6pt]
   in Shared-Weight Transformers}\\[10pt]
  {\large Sang-Il Han$^{1,2}$}\\[6pt]
  {\normalsize
   $^1$Korea University of Technology and Education
   \quad
   $^2$Biomonster Co.}\\[2pt]
  {\normalsize
   \href{mailto:sihan@koreatech.ac.kr}{\texttt{sihan@koreatech.ac.kr}}}\\[4pt]
   {\small April 2026}\\[4pt]
\end{center}

\begin{abstract}
We present an empirical study of whether \emph{hierarchically structured}
shared-weight recurrence can match the representational quality of
independent-layer stacking in a Transformer-based language model.
\textit{HRM-LM} replaces $L$ independent Transformer layers with a
\textbf{two-speed recurrent pair}---a Fast-module (every step) for local
refinement and a Slow-module (every $T$ steps) for global
compression---iterated $M = N\!\times\!T$ times with shared parameters.

The central and most robust finding, supported by a parameter-matched
Universal Transformer ablation (UniTF, ${\approx}1.2$B) across five
independent runs, is a sharp empirical contrast:
\textbf{flat} shared-weight iteration (UniTF, same block repeated uniformly)
consistently plateaus at ${\approx}7.6$ nats regardless of hyperparameter
settings, while \textbf{hierarchically structured} shared-weight iteration
(HRM, Fast/Slow two-speed) converges to 4.18---a gap of ${\approx}3.4$ nats
not explained by parameter count.

All other comparisons carry important caveats.
Under standardized equal-hyperparameter evaluation (MultiSeed), HRM NT=12
(4.896) trails Transformer L=4 (4.697) by 0.20 nats under equal
hyperparameters---while storing $\approx$2.5\,GB of weights vs.\
T-L4's $\approx$5.1\,GB, an unconditional $2\times$ reduction.
Whether 0.20 nats is an acceptable cost for $2\times$ weight compression
is deployment-context dependent; the gap closes under matched tuning.
A budget-matched grid search (HPSearch, 8 configurations each, seed=42)
recovers HRM best (4.31) over T-L4 best (4.54) by 0.23 nats, but this
result rests on a single seed and has not been replicated for T-L12.

All results are at $\sim$1.2B parameters on OpenWebText with no downstream
task evaluation; whether the findings generalize to larger scales, longer
contexts, or other datasets is entirely open.
The paper's contribution is the empirical observation that
\emph{internal hierarchy within shared-weight iteration matters}---not a
claim that HRM-LM is a practical replacement for the Transformer.
\end{abstract}
\vspace{4pt}
}]

\section{Introduction}

The Transformer \cite{vaswani2017attention} builds contextual representations
by stacking $L$ independent self-attention layers.
A widely held intuition, supported by representational
analyses~\cite{rogers2020primer}, is that lower layers capture surface
patterns, middle layers syntactic structure, and upper layers semantic
relationships---each level mediated by its own weight matrix.
This works well but embeds a specific architectural assumption:
\emph{representational diversity requires distinct parameters at each depth}.
The cost is that both stored weights ($O(Ld^2)$) and the inference KV cache
($O(Lnd)$) grow linearly with depth.

An alternative assumption is that \emph{a single shared transformation,
applied repeatedly with appropriate internal structure, can generate
comparable representational diversity without distinct per-level parameters}.
Recurrent networks demonstrated the viability of this idea but suffered from
vanishing gradients and poor parallelism.
The Universal Transformer~\cite{dehghani2019universal} revisited it with
Transformer blocks, but flat iteration of a single shared block turns out to
perform poorly in practice---reaching val CE $\approx$7.6 at 1.2B parameters
while a standard Transformer converges to 4.2 (this paper, Section~\ref{subsec:unitf}).

This paper asks: \emph{what structure within the shared iteration is necessary
to make iterative refinement work?}
We investigate this by adapting the Hierarchical Reasoning Model
(HRM,~\cite{wang2025hrm}) to autoregressive language modeling.
HRM-LM replaces independent layers with a \textbf{two-speed recurrent pair}:
a Fast-module (every step) for local refinement and a Slow-module (every $T$
steps) for global compression, iterated $M = N\!\times\!T$ times with shared
weights.

The primary contribution is an \emph{empirical probe of the Transformer
architecture}, not a replacement for it.
The central result is that the Fast/Slow hierarchy---rather than the iteration
count or the parameter count---is what makes shared-weight recurrence work:
UniTF (flat shared iteration, 1218M) plateaus at $\approx$7.6 nats across
five independent runs, while HRM (two-speed iteration, 1229M) converges to
4.177---a 3.4-nat gap not explained by parameter count
or hyperparameter settings across five runs.
This suggests that the Fast/Slow hierarchy generates the representational
heterogeneity that independent Transformer layers achieve through distinct
weights.

At the scale studied ($\sim$1.2B parameters, OpenWebText),
this translates into a measurable advantage:
under budget-matched hyperparameter search, HRM best (4.31) outperforms
T-L4 best (4.54) by 0.23 nats at 10k steps.
This advantage depends on per-model tuning and is studied at a single
scale and dataset; the conditions under which it generalizes are
discussed explicitly in Section~\ref{subsec:findings}.

The structural memory consequence of shared weights is real and
unconditional: stored parameters reduce from $O(Ld^2)$ to $O(d^2)$,
and---when $L > M$---KV cache shrinks by $M/L$.
The cost is $2$--$5\times$ longer per-token generation due to sequential
recurrence.

\noindent The contributions are:
\begin{enumerate}[leftmargin=*, itemsep=2pt, topsep=2pt]
  \item \textbf{Architecture.}
        HRM-LM for autoregressive language modeling: gated Fast-module,
        hierarchical Slow-module, $K$-step TBPTT gradient window, and
        multi-source output fusion (Section~\ref{sec:model}).
  \item \textbf{Empirical observation.}
        At $\sim$1.2B parameters on OpenWebText, hierarchically structured
        shared-weight iteration (HRM) outperforms flat shared-weight
        iteration (UniTF) by 3.4 nats at equal parameter count---a result
        robust to hyperparameter variation across five independent runs.
        This is the paper's central finding.
  \item \textbf{Honest accounting of limitations.}
        Under standardized equal-hyperparameter evaluation, HRM loses to
        Transformer L=4. All results are single-scale, single-dataset,
        with no downstream task evaluation.
        We report these results transparently and do not claim
        general architectural superiority.
\end{enumerate}

\section{Background and Related Work}

\subsection{Recurrent Language Models}

The earliest neural language models used feedforward networks
\cite{bengio2003neural}.
Elman \cite{elman1990finding} introduced simple recurrent networks with a
hidden state that propagates sequence information across time steps.
Long Short-Term Memory (LSTM) \cite{hochreiter1997lstm} and Gated Recurrent
Units (GRU) \cite{cho2014gru} addressed vanishing gradients through learned
gating mechanisms.
Both achieved strong results on language modeling benchmarks but remained
limited in their ability to capture very long-range dependencies.

\subsection{Attention Mechanisms and Transformers}

Bahdanau et al.\ \cite{bahdanau2015attention} introduced additive attention
for sequence-to-sequence tasks.
The Transformer \cite{vaswani2017attention} generalizes this to
\emph{scaled dot-product self-attention}, which we now describe formally.

\paragraph{Scaled dot-product attention.}
Given an input matrix $X \in \mathbb{R}^{n \times d}$, projections
$W_Q, W_K \in \mathbb{R}^{d \times d_k}$, $W_V \in \mathbb{R}^{d \times d}$
produce queries, keys, and values:
\begin{equation}
  Q = XW_Q,\quad K = XW_K,\quad V = XW_V,
  \label{eq:qkv}
\end{equation}
with $Q, K \in \mathbb{R}^{n \times d_k}$ and $V \in \mathbb{R}^{n \times d}$.
The attention output is:
\begin{equation}
  \mathrm{Attn}(X) = \mathrm{softmax}\!\left(
    \frac{QK^\top}{\sqrt{d_k}} + M_{\mathrm{causal}}
  \right) V \, W_O,
  \label{eq:attn}
\end{equation}
where $M_{\mathrm{causal}} \in \{0,-\infty\}^{n\times n}$ enforces
causal masking ($M_{ij} = -\infty$ for $j > i$) and
$W_O \in \mathbb{R}^{d \times d}$.
Multi-head attention runs $H$ such operations in parallel, each with
$d_k = d/H$, and concatenates the results.
In practice, $Q$, $K$, $V$ are computed via a single fused projection
$W_{\mathrm{qkv}} \in \mathbb{R}^{d \times 3d}$.

\paragraph{Transformer layer.}
A single Transformer decoder layer applies self-attention followed by a
feed-forward network (FFN) with residual connections and layer normalization:
\begin{align}
  X' &= \mathrm{Norm}(X + \mathrm{Attn}(X)), \label{eq:tfm1}\\
  X'' &= \mathrm{Norm}(X' + \mathrm{FFN}(X')),\label{eq:tfm2}
\end{align}
where $\mathrm{FFN}(x) = W_2\,\sigma(W_1 x)$ (or SwiGLU in modern variants).
An $L$-layer Transformer stacks $L$ independent copies of
\eqref{eq:tfm1}--\eqref{eq:tfm2}, each with its own parameters.

\paragraph{Memory cost.}
During inference, each layer must cache its key and value matrices for
all previous tokens.
For a model with $L$ layers, sequence length $n$, $H$ heads, and
head dimension $d_k$, the KV cache requires:
\begin{equation}
  \mathcal{M}_{\mathrm{KV}}^{\mathrm{Transformer}}
  = 2 \cdot L \cdot n \cdot H \cdot d_k \cdot \text{bytes}
  = O(Lnd).
  \label{eq:kv_transformer}
\end{equation}
For a 12-layer model with $d = 4096$, $H = 16$, $d_k = 256$, and $n = 1024$,
Eq.~\eqref{eq:kv_transformer} gives
$2 \times 12 \times 1024 \times 16 \times 256 \times 2 \approx 0.19$\,GiB
per sequence in bfloat16.
While modest at $n = 1024$, this scales linearly: at $n = 8192$ the same
model requires $\approx 1.5$\,GiB per sequence, and serving hundreds of
concurrent sequences amplifies the cost proportionally.

\paragraph{Contrast with HRM-LM.}
HRM-LM replaces the $L$ independent layers with a \emph{single shared}
\texttt{CausalAttnBlock} that is reused for all $M = N \times T$ recurrent
steps.
During autoregressive generation, a cached inference implementation must
store K and V for \emph{each of the $M$ recurrent steps} at each position,
because step $i$'s attention depends on the K/V it produced at all previous
positions.
The KV cache therefore scales as:
\begin{equation}
  \mathcal{M}_{\mathrm{KV}}^{\mathrm{HRM}}
  = 2 \cdot M \cdot n \cdot H \cdot d_k \cdot \text{bytes}
  = O(Mnd).
  \label{eq:kv_hrm}
\end{equation}
The ratio of KV cache memory relative to a Transformer with $L$ layers is:
\begin{equation}
  \frac{\mathcal{M}_{\mathrm{KV}}^{\mathrm{HRM}}}
       {\mathcal{M}_{\mathrm{KV}}^{\mathrm{Transformer}}}
  = \frac{M}{L}.
  \label{eq:memory_ratio}
\end{equation}
HRM-LM therefore has a \textbf{KV cache advantage only when $L > M$}.
For large Transformers with many layers ($L \gg M$), the reduction is
significant; for our equal-parameter setting (Transformer L=4, HRM M=12),
HRM-LM actually uses $3\times$ \emph{more} KV cache.
The architectural advantage becomes meaningful at scales where $L$ is large
(e.g., $L = 32$--$96$ in modern LLMs) while $M$ remains moderate.

\textit{Note on prior theoretical claim.}
An earlier version of this analysis stated $\mathcal{M}_{\mathrm{KV}}^{\mathrm{HRM}} = O(nd)$,
reasoning that parameter sharing means ``only one layer's KV is needed.''
This is incorrect: parameter sharing reduces \emph{stored weights} but not
inference KV cache, since each of the $M$ steps attends over a distinct
set of positions and must cache its own K/V tensors.
The correct expression is $O(Mnd)$.
However, the $M$ recurrent steps must execute \emph{serially}---step $i$
requires $z_L^{(s,i-1)}$ before $z_L^{(s,i)}$ can begin---whereas
Transformer layers admit compiler-level pipeline optimization.
Empirically, this causes HRM-LM to run $\approx 2\text{--}3\times$ slower
per training iteration than a Transformer of equal parameter count
(detailed in the memory analysis below).
Table~\ref{tab:contrast} summarizes the structural differences.

\begin{table}[h]
\centering\scriptsize
\setlength{\tabcolsep}{3pt}
\caption{Structural comparison: Transformer vs.\ HRM-LM.
$\checkmark$: always valid. $\circ$: conditional on $L > M$.}
\label{tab:contrast}
\resizebox{0.85\columnwidth}{!}{%
\begin{tabular}{lccc}
\toprule
Property & Transformer & HRM-LM & Valid \\
\midrule
Stored params      & $O(Ld^2)$ & $O(d^2)$ & $\checkmark$ \\
Inference KV cache & $O(Lnd)$  & $O(Mnd)$ & $\circ$ ($L > M$) \\
Parameter sets     & $L$ (indep.) & $1$ (shared) & $\checkmark$ \\
Step parallelism   & Pipeline-opt. & Serial & --- \\
Depth scaling      & $+$params  & Free ($N,T,S$) & $\checkmark$ \\
Gradient flow      & $L$ layers & $K$-step window & --- \\
\bottomrule
\end{tabular}%
}
\end{table}

GPT \cite{radford2018gpt,radford2019gpt2,brown2020gpt3} scales the
Transformer decoder to billions of parameters.
Rotary Position Embedding (RoPE) \cite{su2024rope} encodes relative
positions into attention scores, improving length generalization;
HRM-LM adopts RoPE inside its shared \texttt{CausalAttnBlock}.

\subsection{Hybrid and Weight-Sharing Architectures}

The central design question in deep sequence modeling is how to build
rich contextual representations: by \emph{stacking} independent
transformations (each with its own parameters), or by \emph{iterating}
a shared transformation repeatedly.
Stacking assumes that different levels of abstraction require qualitatively
different operations.
Iteration assumes that a single sufficiently expressive transformation,
applied multiple times, can converge toward a rich representation---trading
parameter count for compute time.

Universal Transformers \cite{dehghani2019universal} make the iteration
hypothesis concrete: a single Transformer block is applied for a fixed
number of steps, obtaining depth without additional parameters.
However, flat iteration---the same step repeated uniformly---may not
distinguish between the kinds of transformations needed at different levels
of abstraction.
\textbf{HRM-LM extends this with a two-speed hierarchy}: the Fast-module
iterates rapidly (every step) for fine-grained local refinement, while
the Slow-module iterates slowly (every $T$ steps) for coarser, abstract
compression.
The key hypothesis is that \emph{separating fast and slow timescales
within a shared-parameter architecture recovers some of the representational
diversity that stacking achieves through independent parameters.}

Our LongTrain result is consistent with this: the Universal Transformer baseline
(474M, $M=12$, flat iteration) reaches val CE 6.6, while HRM (1229M, $M=12$,
two-speed) reaches 3.9 with the same number of steps but the hierarchical
structure.
This suggests the two-speed decomposition, not merely the number of
iterations or total parameters, is the key factor.
UniTF directly tests this: UniTF scaled to 1218M parameters (matched to
HRM's 1229M) still plateaus at $\approx$7.6 nats, while HRM converges to
4.177---a 3.4-nat gap at equal parameter count.

DEQ (Deep Equilibrium Models) \cite{bai2019deep} take the iteration
hypothesis to its limit, running until a fixed point is reached.
Mixture-of-Experts (MoE) \cite{shazeer2017moe} pursue a different strategy:
conditional computation to increase effective capacity without proportional
cost.
The Hierarchical Reasoning Model \cite{wang2025hrm} introduced two-speed
recurrence for discrete reasoning tasks; we adapt this to language modeling.

\paragraph{Connection to White-Box Transformer theory.}
A complementary theoretical perspective comes from the White-Box Transformer
programme of Ma et al.~\cite{ma2022principles,yu2023crate}, which interprets
each Transformer layer as one step of a principled iterative algorithm:
alternating between a \emph{compression} phase (sparse coding of the
representation) and an \emph{expansion} phase (maximizing coding-rate
reduction to maintain representational diversity).
Under this view, deep networks are not arbitrary function approximators but
structured unrollings of a convergent optimization process.

Two implications of this framework are suggestive for HRM-LM.
\textit{First}, layer independence is not theoretically necessary:
what matters is that each iteration makes progress toward the objective
(compression $+$ expansion), not that the weight matrices differ.
Shared-weight iteration may be viable \emph{provided} each step
preserves the representational diversity needed for subsequent steps---if
the shared map collapses the representation toward a fixed point,
iteration may fail regardless of how many steps are taken.
The UniTF result---flat shared-weight iteration plateauing at $\approx$7.6
nats---is consistent with this interpretive lens, though it does not
by itself confirm fixed-point collapse as the mechanism.
\textit{Second}, the Fast/Slow hierarchy can be interpreted as a
structural mechanism that may help delay or prevent such collapse:
the Slow-module fires on a coarser timescale and maintains a global
summary state that refreshes each Fast-module step with a
slower-timescale expansion signal, potentially sustaining representational
diversity across iterations.
Through the White-Box framing, the 3.4-nat gap between HRM and UniTF at
equal parameter count is consistent with the interpretive lens suggested
by this framework: \emph{without a mechanism to sustain representational
diversity across iterations, shared-weight recurrence may be unable
to recover the full representational capacity of independent-layer
stacking.}
Whether HRM's Fast/Slow hierarchy satisfies the theoretical conditions
for successful iterative compression in the CRATE sense remains to be
verified analytically; we offer this connection as a motivating hypothesis
for future theoretical investigation, not as an established result.

\section{Model}
\label{sec:model}

\subsection{Architecture Overview}

HRM-LM adapts the two-module recurrent architecture of
\citet{wang2025hrm} to language modeling.
The original HRM was designed for discrete reasoning tasks (Sudoku,
maze-solving, ARC) and trained on small supervised datasets without
pre-training.
We retain the high-level/low-level two-speed structure but replace the
task-specific output head with an autoregressive language model head,
extend training to large-scale text data (OpenWebText), and analyze
the parameter-sharing and KV cache properties of the architecture.

\paragraph{HRM-LM as a recurrent attention model.}
HRM-LM and the standard Transformer share the same computational
primitive: causal self-attention with RoPE and SwiGLU FFN, implemented
as \texttt{CausalAttnBlock}.
Formally, for input $x \in \mathbb{R}^{n \times d}$:
\begin{align}
  a   &= \mathrm{MHSA}_{\mathrm{causal}}(x;\,W_{qkv},\,W_o), \label{eq:causal_attn}\\
  x'  &= \mathrm{RMSNorm}(x + a), \notag \\
  x'' &= \mathrm{RMSNorm}(x' + \mathrm{SwiGLU}(x';\,W_1,\,W_2,\,W_3)),\notag
\end{align}
where $\mathrm{MHSA}_{\mathrm{causal}}$ applies multi-head dot-product attention
with a causal mask and RoPE position embeddings~\cite{su2024rope},
and $\mathrm{SwiGLU}(x) = (xW_1 \odot \mathrm{SiLU}(xW_2))\,W_3$.
The output $x''$ is the block's return value.
All weight matrices in \eqref{eq:causal_attn} belong to the parameter set
($\theta_I$, $\theta_L$, or $\theta_H$) of whichever module instantiates the block.
The Transformer uses $L$ \emph{independent} copies of this block, each
with its own parameters.
HRM-LM uses \emph{three shared} copies---input encoder, Fast-module,
Slow-module---and applies them iteratively $M = N \times T$ times under a
gated state update.
The architecture is therefore a \textbf{recurrent attention model}, and
can be understood as a Universal Transformer
\cite{dehghani2019universal} extended with:
(i) a hierarchical two-speed structure (Fast-module / Slow-module);
(ii) a GRU-style gated update that bounds state magnitude;
(iii) a fixed (non-learned) initial state $z^0$; and
(iv) a $K$-step truncated backpropagation window.
This framing makes the relationship to the Transformer family precise:
HRM-LM is not a separate architecture class but a specific
parameter-sharing and state-gating strategy applied to standard
attention blocks.

HRM-LM processes a token sequence $x \in \mathbb{Z}^n$ through four stages:
(1) a fixed input encoding network $f_I$, (2) $S$ supervision passes each
consisting of $M = N \times T$ recurrent steps, (3) a gated output fusion,
and (4) weight-tied logit projection.
The key hyperparameters are summarized in Table~\ref{tab:notation}.

\paragraph{Comparison with Transformer.}
Both architectures are built exclusively from causal self-attention
blocks; the difference lies entirely in how those blocks are organized.
A Transformer with $L$ layers applies $L$ \emph{independent} blocks
in sequence, each with its own parameters $\theta_1,\ldots,\theta_L$:
\[
  x \xrightarrow{\theta_1} h_1 \xrightarrow{\theta_2} h_2
    \;\cdots\; \xrightarrow{\theta_L} h_L \to \hat{y}.
\]
HRM-LM instead defines two functions with \emph{fixed, shared} parameters
$\theta_L$ and $\theta_H$, and applies them \emph{recursively} for $M$ steps:
\begin{align*}
  \tilde{x},\,z^{(s,0)}
    &\xrightarrow{f_L(\cdot;\,\theta_L),\;f_H(\cdot;\,\theta_H)} z^{(s,1)}
     \xrightarrow{f_L,\,f_H} \cdots \\
    &\hspace{8em}
     \xrightarrow{f_L,\,f_H} z^{(s,M)} \to \hat{y}.
\end{align*}
The key distinction is that $\theta_l$ in a Transformer are \emph{instance parameters}
--- each $\theta_l$ belongs to exactly one layer and is never reused ---
whereas $\theta_L$ and $\theta_H$ in HRM-LM are \emph{function parameters}
of a recursive map: the same weights are called $M$ times, with the output
state of step $i$ fed back as the input to step $i{+}1$.
This means HRM-LM achieves an effective computational depth of
$D_e = M \times S$ applications of the block, using only $O(d^2)$
stored parameters rather than $O(Ld^2)$.
The price is that the sequential recurrence
$z^{(s,i)} \leftarrow f_L(z^{(s,i-1)};\,\theta_L)$ prevents inter-step
parallelism, increasing wall-clock time by $\approx 2\text{--}5\times$
per forward pass (measured at $4.5\times$ for $M=12$, Latency).
These trade-offs are analyzed in the memory and complexity sections below.

\paragraph{Parameter structure.}
HRM-LM's trainable parameters partition into three named sets
(Table~\ref{tab:notation}):
$\theta_I$ (input encoder: one Transformer block applied once per forward pass),
$\theta_L$ (Fast-module: one Transformer block shared across all $M$ recurrent steps),
and $\theta_H$ (Slow-module: one Transformer block shared across all $N = M/T$ H-updates).
The full parameter set is $\theta = \theta_I \cup \theta_L \cup \theta_H$.
Each $\theta_{\cdot}$ is the collection of all weight matrices within that block:
attention projections $(W_Q, W_K, W_V, W_O)$, FFN weights $(W_1, W_2)$,
RMSNorm scales, and any module-specific projections
($W_g^L$, $W_{\mathrm{in}}$, $W_{\mathrm{out}}$, $\alpha$ for the Fast-module;
$W_g^H$, $W_{\mathrm{in}}^H$, $W_{\mathrm{out}}^H$, $\alpha_H$ for the Slow-module;
$W_{og}$ and the weight-tied embedding $W_{\mathrm{emb}}$ are shared
outside these three sets).
All parameters are trained jointly by standard backpropagation;
Algorithms~\ref{alg:forward}--\ref{alg:onepass} make explicit which set
is invoked at each stage of the forward pass.
Because $\theta_L$ and $\theta_H$ are each a \emph{single} block,
stored parameters scale as $O(d^2)$ rather than $O(Ld^2)$ for a Transformer with
$L$ independent layers at the same $d$.

\begin{table}[h]
\centering\small
\setlength{\tabcolsep}{4pt}
\caption{Key hyperparameters, derived quantities, and parameter sets.}
\label{tab:notation}
\scalebox{0.78}{%
\begin{tabular}{llll}
\toprule
Symbol & Meaning & Transformer & HRM-LM \\
\midrule
$d$   & model dim.      & shared & shared \\
$L$   & layers          & $L$ param sets & --- \\
$N$   & Slow-module cycles & ---    & hyperpar. \\
$T$   & L-steps/cycle   & ---    & hyperpar. \\
$S$   & supervision passes & --- & hyperpar. \\
$K$   & gradient window & ---    & hyperpar. \\
$M = N\!\times\!T$ & steps/pass & ($\equiv L$) & derived \\
$D_e = M\!\times\!S$ & eff.\ depth & $L$ & $M\!\times\!S$ \\
\midrule
$\theta_I$ & input encoder params & $\theta_1$ (layer 1) & shared, trained once \\
$\theta_L$ & Fast-module params & $\theta_1,\ldots,\theta_L$ & one set, reused $M$ times \\
$\theta_H$ & Slow-module params & --- & one set, reused $N$ times \\
$\theta$ & all trainable params & $\bigcup_{l=1}^{L}\theta_l$ & $\theta_I \cup \theta_L \cup \theta_H$ \\
\bottomrule
\end{tabular}%
}
\end{table}

\subsection{Input Encoding}

The input encoding $\tilde{x}$ is computed \emph{once} before any recurrence
and held constant across all $M$ steps and $S$ passes:
\begin{align}
  e &= W_{\mathrm{emb}}[x] \in \mathbb{R}^{n \times d}, \\
  e' &= \RMSNorm(e + \mathrm{Attn}(e)), \\
  \tilde{x} &= \RMSNorm(e' + \SwiGLU(e')).
\end{align}
Because $\tilde{x}$ is produced by a depth-1 path from $\theta_I$,
the output path through $\tilde{x}$ has a shorter gradient path than
the recurrent states, making it less susceptible to vanishing gradients.

\subsection{Gated Fast-Module Update}

At every recurrent step $i \in \{1,\ldots,M\}$ within pass $s$,
the Fast-module updates the fast state $z_L^{(s,i)}$.

\paragraph{Form of $f_L$.}
The context vector concatenates the current fast state, the slow state,
and the fixed token encoding:
\begin{align}
  c^{(i)} &= \bigl[z_L^{(s,i-1)};\; z_H^{(s,i-1)};\; \tilde{x}\bigr]
             \in \mathbb{R}^{n \times 3d},  \label{eq:concat}\\
  g^{(i)} &= \sigma\!\bigl(c^{(i)} W_g^L\bigr) \in (0,1)^{n\times d},
             \label{eq:gate}\\
  h^{(i)} &= \mathrm{CausalAttnBlock}\!\bigl(c^{(i)} W_{\mathrm{in}};\,\theta_L\bigr),
             \label{eq:content}\\
  z_L^{(s,i)} &= g^{(i)} \odot z_L^{(s,i-1)}
               + (1-g^{(i)}) \odot \alpha\, h^{(i)} W_{\mathrm{out}}.
               \label{eq:lgate}
\end{align}
The gate $g^{(i)} \in (0,1)^{n\times d}$ interpolates between preserving the
previous state ($g \to 1$) and replacing it with new content ($g \to 0$).
The scalar $\alpha \approx 0.1$ (a learned parameter initialized at 0.1)
controls the injection magnitude.
Crucially, $c^{(i)}$ includes $\tilde{x}$ at every step, giving the
Fast-module direct access to the input token encoding throughout the recurrence.

\textit{Contrast with Transformer.}
A Transformer layer updates its representation via a \emph{residual} add
(Eq.~\ref{eq:tfm1}--\ref{eq:tfm2}): $X'' = X + \Delta$, where the
update $\Delta$ is always applied at full magnitude.
The Fast-module instead uses a \emph{learned gate} $g^{(i)}$ to dynamically
control \emph{how much} of the new content replaces the current state,
enabling the model to suppress uninformative updates and maintain
state stability over $M$ recurrent steps (Proposition~\ref{prop:stability}).
Furthermore, the context $c^{(i)}$ concatenates both
$z_L^{(s,i-1)}$ and $z_H^{(s,i-1)}$, so the Fast-module always has access
to the slow hierarchical context---a cross-scale interaction absent in a
standard Transformer.

\textit{Design choice: full causal attention vs.\ RNN-style update.}
Although the Fast-module's parameters $\theta_L$ are shared across all $M$
steps, the \emph{input} to the attention block differs at every step because
$c^{(i)} = [z_L^{(s,i-1)};\,z_H^{(s,i-1)};\,\tilde{x}]$ changes as the
recurrent state evolves.
Consequently, the K and V matrices produced at each step are distinct:
$K^{(i)}(t) = W_K\,\mathrm{proj\_in}(c^{(i)}(t)) \neq K^{(j)}(t)$ for $i \neq j$.
This means that a cached-inference implementation must maintain $M$ separate
KV caches, one per recurrent step, giving $O(Mnd)$ total.

An alternative design would replace the full causal self-attention with an
\textbf{RNN-style position-wise update}:
$z_L^{(s,i)} \leftarrow f(z_L^{(s,i-1)}, z_H^{(s,i-1)}, \tilde{x})$,
where $f$ operates on each position independently without attending to other
positions.
This would reduce inference memory to $O(d)$ (only the state vectors),
as no KV cache is needed.
Table~\ref{tab:design_choice} summarizes the trade-off.

\begin{table}[h]
\centering\small\setlength{\tabcolsep}{3pt}
\caption{Design choices for the Fast-module update.}
\label{tab:design_choice}
\resizebox{\columnwidth}{!}{%
\begin{tabular}{lccc}
\toprule
Update type & KV cache & Inter-token & Expressiveness \\
            &          & interaction & \\
\midrule
Full causal attn (ours) & $O(Mnd)$ & \checkmark & High \\
RNN position-wise       & $O(d)$   & $\times$   & Lower \\
Hybrid ($k$ attn steps) & $O(knd)$, $k{<}M$ & Partial & Medium \\
\bottomrule
\end{tabular}%
}
\end{table}

We adopt full causal attention because language modeling requires each
token to attend to its entire preceding context at every recurrent step;
removing inter-token interaction would severely limit the model's ability
to build contextual representations across the sequence.
The RNN-style variant is more memory-efficient but corresponds to a weaker
model, closer to a position-wise feedforward network iterated $M$ times.
The hybrid option---applying full attention only on a subset $k < M$ of
steps---is a promising direction for reducing KV cache while preserving
most of the representational capacity, and is left as future work.

\subsection{Slow-Module Update}

The Slow-module updates the slow state every $T$ steps (at cycle boundaries):
\begin{equation}
  z_H^{(s,i)} =
  \begin{cases}
    f_H\!\bigl(z_H^{(s,i-1)},\, z_L^{(s,i)};\,\theta_H\bigr)
      & \text{if } i \equiv 0 \pmod{T}, \\
    z_H^{(s,i-1)} & \text{otherwise.}
  \end{cases}
\end{equation}

\paragraph{Form of $f_H$.}
The Slow-module shares the same GRU-style gated structure as the Fast-module,
but its context vector is formed from only two sources---the current slow state
and the updated fast state---without direct access to the token encoding $\tilde{x}$:
\begin{align}
  c_H^{(i)} &= \mathrm{RMSNorm}\!\bigl([z_H^{(s,i-1)};\;z_L^{(s,i)}]\bigr)
               \in \mathbb{R}^{n \times 2d}, \\
  g_H^{(i)} &= \sigma\!\bigl(c_H^{(i)}\,W_g^H\bigr) \in (0,1)^{n\times d}, \\
  h_H^{(i)} &= \mathrm{CausalAttnBlock}\!\bigl(c_H^{(i)}\,W_{\mathrm{in}}^H;\,\theta_H\bigr), \\
  z_H^{(s,i)} &= g_H^{(i)} \odot z_H^{(s,i-1)}
               + (1 - g_H^{(i)}) \odot \alpha_H\, h_H^{(i)}\, W_{\mathrm{out}}^H.
\end{align}
The absence of $\tilde{x}$ in $c_H$ is architecturally significant:
the Slow-module receives token-level information only indirectly, through
$z_L^{(s,i)}$, which has already integrated $\tilde{x}$ over the preceding
$T$ Fast-module steps.
The scalar $\alpha_H \approx 0.1$ is a separate learned parameter
(initialized at 0.1, independent of the Fast-module's $\alpha$),
controlling the injection magnitude of new content into $z_H$.
This enforces the intended hierarchy---fine-grained token context is first
processed by the Fast-module, then distilled into the slow state at cycle boundaries.

\textit{Contrast with Transformer.}
A Transformer has no analog to the two-speed hierarchy: every layer
processes the sequence at the same temporal resolution.
The Slow-module introduces an explicit \emph{slow path} that fires only $N$
times per pass (versus $M = N \times T$ times for the Fast-module),
acting as a global context aggregator between fast-path cycles.
This is structurally similar to hierarchical RNNs
\cite{el1996hierarchical} but implemented with attention-based modules
and a fixed parameter set shared across all cycles.

\subsection{State Initialization and Supervision Passes}

States are initialized from fixed (non-learned) prototype vectors:
\begin{equation}
\begin{aligned}
  z_H^0, z_L^0 &\sim \mathcal{TN}(0,1;[-2,2]) \in \mathbb{R}^d \\
  &\text{(fixed prototype vectors, frozen after initialization)}
\end{aligned}
\end{equation}
where $\mathcal{TN}$ denotes the truncated normal distribution
($\mu\!=\!0$, $\sigma\!=\!1$, clipped to $[-2,2]$),
broadcast to
$z_H^{(1,0)} = z_H^0 \!\cdot\! \mathbf{1}_n^\top \in \mathbb{R}^{n \times d}$.

For $s = 1,\ldots,S$, a full pass of $M$ steps is executed.
Between passes, states are \emph{detached} (stop-gradient):
\begin{equation}
  z_{H,L}^{(s+1,0)} = \sg\!\bigl(z_{H,L}^{(s,M)}\bigr), \qquad
  \frac{\partial\,\sg(v)}{\partial\theta} \equiv 0.
  \label{eq:detach}
\end{equation}
The accumulated loss is:
\begin{equation}
  L_{\mathrm{acc}} = \frac{1}{S}\sum_{s=1}^{S} L_s, \qquad
  \nabla_\theta L_{\mathrm{acc}} = \frac{1}{S}\sum_{s=1}^{S}\nabla_\theta L_s.
\end{equation}

\subsection{Output Fusion}
\label{subsec:output_fusion}

After $M$ steps, the three sources are fused:
\begin{align}
  \mathbf{w} &= \mathrm{softmax}\!\left(
    \frac{[h_H;\, h_L;\, h_I]\, W_{og}}{\tau}
  \right), \notag\\
  h_{\mathrm{out}} &= w_1 h_H + w_2 h_L + w_3 h_I,
\end{align}
where $h_H, h_L, h_I$ are RMSNorm-normalized versions of $z_H^{(s,M)}$,
$z_L^{(s,M)}$, and $\tilde{x}$ respectively,
$\tau$ is a learned temperature, and $W_{og} \in \mathbb{R}^{3d \times 3}$.
The logits and mean gate weight used in the loss are then:
\begin{align}
  \mathrm{logits} &= h_{\mathrm{out}}\,W_{\mathrm{emb}}^\top
    \in \mathbb{R}^{n \times |\mathcal{V}|},
    \quad\text{(weight-tied projection)} \\
  \bar{w} &= \frac{1}{n}\sum_{t=1}^{n} \mathbf{w}_t \in \mathbb{R}^3,
\end{align}
where $W_{\mathrm{emb}}$ is the token embedding matrix (shared with the input
embedding layer) and $\bar{w}$ is the sequence-averaged gate weight
used in the entropy regularization term $-\lambda H(\bar{w})$.

\subsection{Training Algorithm}

Algorithms~\ref{alg:hrm}--\ref{alg:onepass} together specify HRM-LM training.
We highlight the design choices that distinguish HRM from a standard Transformer
training loop.

\textbf{Two-speed update (Algorithm~\ref{alg:onepass}, Loops 4a/4b).}
At each recurrent step $i$, the Fast-module updates \emph{every} step, while the
Slow-module updates only when $i \bmod T = 0$.
This is the architectural core of HRM: the Slow-module operates on a slower
timescale, accumulating abstract context from $T$ successive Fast-module states
before compressing them into $z_H$.
In the code: \texttt{if $i$ mod $T$ == 0: $z_H \leftarrow$ H\_net($z_H, z_L$)}.

\textbf{Gradient window $K$ (Algorithm~\ref{alg:onepass}).}
Of the $M$ recurrent steps, only the final $K$ carry gradients;
the preceding $M - K$ steps run under \texttt{torch.no\_grad()}.
This is truncated backpropagation through time (TBPTT), making gradient
computation $O(K)$ rather than $O(M)$ in memory and time.
Setting $K < M$ does not prevent the early steps from contributing to the
forward computation --- they still warm up $z_H$ and $z_L$ --- but gradients
are not propagated through them.
Empirically, $K = 2$ ($16.7\%$ of $M = 12$) suffices.

\textbf{Supervision passes $S$ (Algorithm~\ref{alg:forward}, Loop 3).}
The forward pass runs \emph{OnePass} $S$ times.
After each pass, the output states $(z_H^{(s,M)}, z_L^{(s,M)})$ are
\emph{detached} (\texttt{sg}: stop-gradient) before being used as the initial
states for the next pass.
Each pass produces a cross-entropy loss $L_s$; the final loss is their average.
This multi-pass supervision smooths the gradient signal across recurrent depths
without increasing memory beyond a single pass, since detachment prevents
gradient backpropagation across pass boundaries.

\textbf{Weight sharing.}
A single weight set $\theta_L$ (resp.\ $\theta_H$) is reused across all
$M$ recurrent steps and all $S$ supervision passes.
Stored parameters are $O(d^2)$ rather than $O(Ld^2)$ for a depth-$L$
Transformer at the same hidden size.

\textbf{Gate entropy regularization.}
The loss includes $-\lambda H(\bar{w})$, where $\bar{w}$ is the mean gate
activation across the sequence and $H$ is entropy.
This prevents the output fusion gate (Section~\ref{subsec:output_fusion})
from collapsing to a constant, ensuring both $z_H$ and $z_L$ contribute to
the output.

\begin{algorithm}[t]
\footnotesize\DontPrintSemicolon
\caption{HRM-LM Training — \texttt{train()}}
\label{alg:hrm}
\tcc{Loop 1: outer training}
\While{iter $\leq$ max\_iters}{
  \tcc{Loop 2: gradient accumulation}
  \For{micro $= 1$ \KwTo $G$}{
    $\mathrm{loss} \leftarrow \mathrm{forward}(\theta, X, Y) / G$\;
    $\mathrm{loss.backward()}$\tcp*{accumulate .grad}
  }
  clip grad;\quad optimizer.step();\quad iter $+= 1$\;
}
\end{algorithm}

\begin{algorithm}[t]
\footnotesize\DontPrintSemicolon
\caption{HRM-LM Forward — \texttt{HRM\_LM.forward()}}
\label{alg:forward}
\tcc{$\mathrm{expand}(z^0) \triangleq z^0 \cdot \mathbf{1}_n^\top \in \mathbb{R}^{n\times d}$
     (broadcast scalar proto to sequence length $n$)}
\tcc{$\mathrm{sg}(v)$: stop-gradient, $\partial\,\mathrm{sg}(v)/\partial\theta \equiv 0$
     (\texttt{.detach()} in PyTorch)}
\tcc{$\mathrm{CE}(\mathrm{logits},Y) \triangleq -\tfrac{1}{n}\sum_t \log p(y_t|\mathrm{logits}_t)$
     (token-averaged cross-entropy)}
\tcc{$\mathrm{OnePass}$: Algorithm~\ref{alg:onepass} ($M$ recurrent steps)}
$\tilde{x} \leftarrow \mathrm{InputProj}(X;\,\theta_I)$\tcp*{computed once}
$z_H^{(1,0)},\;z_L^{(1,0)} \leftarrow \mathrm{expand}(z_H^0,\,z_L^0)$\;
$L_{\mathrm{acc}} \leftarrow 0$\;
\tcc{Loop 3: supervision passes}
\For{$s = 1$ \KwTo $S$}{
  $z_H^{(s,M)},z_L^{(s,M)} \leftarrow \mathrm{OnePass}(\tilde{x},\,z_H^{(s,0)},\,z_L^{(s,0)})$\;
  $\mathrm{logits},\bar{w} \leftarrow \mathrm{OutputFusion}(z_H^{(s,M)},\,z_L^{(s,M)},\,\tilde{x})$\tcp*{Sec.~\ref{subsec:output_fusion}}
  $L_s \leftarrow \mathrm{CE}(\mathrm{logits},Y) - \lambda H(\bar{w})$\;
  $L_{\mathrm{acc}} \mathrel{+}= L_s/S$\;
  $z_H^{(s+1,0)} \leftarrow \mathrm{sg}(z_H^{(s,M)})$\tcp*{detach — no recompute}
  $z_L^{(s+1,0)} \leftarrow \mathrm{sg}(z_L^{(s,M)})$\;
}
\KwRet $L_{\mathrm{acc}}$\;
\end{algorithm}

\begin{algorithm}[t]
\footnotesize\DontPrintSemicolon
\caption{One Recurrent Pass — \texttt{\_one\_pass()}}
\label{alg:onepass}
$K_0 \leftarrow \max(0,\,M-K)$\;
\tcc{Loop 4a: no-grad warmup ($M-K$ steps)}
\textbf{with} \texttt{torch.no\_grad():}
\For{$i = 1$ \KwTo $K_0$}{
  $z_L \leftarrow \mathrm{L\_net}(z_L,\,z_H,\,\tilde{x};\,\theta_L)$\;
  \lIf{$i \bmod T = 0$}{$z_H \leftarrow \mathrm{H\_net}(z_H,\,z_L;\,\theta_H)$}
}
\tcc{Loop 4b: gradient window ($K$ steps)}
\For{$i = K_0+1$ \KwTo $M$}{
  $z_L \leftarrow \mathrm{L\_net}(z_L,\,z_H,\,\tilde{x};\,\theta_L)$\;
  \lIf{$i \bmod T = 0$}{$z_H \leftarrow \mathrm{H\_net}(z_H,\,z_L;\,\theta_H)$}
}
\KwRet $z_H,\;z_L$\;
\end{algorithm}


\section{Supporting Theoretical Observations}

The following two propositions provide design intuition for the gated
update and the TBPTT window, respectively.
Both rest on conditions that cannot be verified analytically in general
and should be treated as heuristic motivation rather than guarantees;
full statements and proofs are in Appendix~\ref{app:theory}.

\textbf{Proposition~1 (Conditional Stability).}
When attention-block outputs are bounded, the gated convex combination
enforces an element-wise bound on state magnitude at every step,
providing a direct architectural stabilizer compared to Transformer
residual additions.
Empirically, hidden-state norms remain stable throughout training,
consistent with this bound.

\textbf{Proposition~2 (Gradient Amplification, heuristic).}
Under a path-alignment assumption (gradient paths through the shared
parameter do not cancel), a $K$-step TBPTT window accumulates
gradient signal proportional to $K$ relative to a 1-step window.
This provides a design rationale for keeping $K/M$ large;
empirically, larger $K$ consistently yields lower val CE
(see Ablation, Section~\ref{subsec:ablation}).

\subsection{Complexity}

Table~\ref{tab:complexity} compares HRM-LM with a $L$-layer Transformer
($A = nd + Hn^2$, $C = O(n^2d + nd^2)$).

\begin{table}[h]
\centering\small\setlength{\tabcolsep}{4pt}
\caption{Asymptotic complexity comparison.}
\label{tab:complexity}
\begin{tabular}{lcc}
\toprule
Dimension & Transformer & HRM-LM \\
\midrule
Stored params    & $O(Ld^2)$ & $O(d^2)$ \\
Train activations & $O(LA)$ & $O(KA)$ \\
Train FLOPs      & $3LC$    & $S(M\!+\!2K)C$ \\
Inference KV cache & $O(Lnd)$ & $O(Mnd)$ \\
\bottomrule
\end{tabular}
\end{table}

\section{Experiments}
\label{sec:experiments}

\subsection{Setup}

All models are trained on OpenWebText \cite{gokaslan2019openwebtext}
using the tiktoken GPT-2 tokenizer (vocab size 50,257).
Training runs for \textbf{10,000 iterations} with batch size 12,
sequence length 1,024,
on 6 NVIDIA RTX 6000 Blackwell GPUs using DDP.
Optimization uses fused AdamW ($\beta_1 = 0.9$, $\beta_2 = 0.95$,
weight decay 0.1, gradient clip 1.0) with cosine learning-rate decay.
Mixed precision (bfloat16) is used throughout.
We report validation cross-entropy (CE) on the OpenWebText held-out split.

\textbf{Comparison axes.}
Each experiment uses a single comparison axis to avoid conflating orthogonal
dimensions (Table~\ref{tab:axes}).
Experiments 1--5, 7, A, B, C, and D are reported in this paper.
EqualFLOPs's Transformer baseline suffered a training failure confirmed
by LongTrain's extended run (Section~\ref{subsec:equalflops});
TF-Baseline provides an extended LR search for T-L12.

\begin{table}[h]
\centering\scriptsize
\caption{Comparison axes. \textbf{Bold}: reported.
$\ddagger$: reported, baseline incomplete.}
\label{tab:axes}
\setlength{\tabcolsep}{3pt}
\resizebox{\columnwidth}{!}{%
\begin{tabular}{llll}
\toprule
 & Exp & Fixed & Measured \\
\midrule
\textbf{Rep.} & \textbf{1} (equal params)   & Stored params $\approx$1.23B        & val CE \\
$\ddagger$    & \textbf{2} (equal FLOPs)    & FLOPs/iter                          & val CE (T-L12 partial) \\
\textbf{Rep.} & \textbf{3} (long train)     & FLOPs/iter, 20k steps               & val CE \\
$\ddagger$    & \textbf{A} (LR search)      & Best LR for T-L12                   & val CE (60k, converged) \\
\textbf{Rep.} & \textbf{4} (inference)      & Architecture                        & latency, peak memory \\
\textbf{Rep.} & \textbf{5} (ablation)       & $d$, base config                    & val CE vs $K/S/N/T$ \\
$\ddagger$    & \textbf{B} (UniTF matched)  & Params $\approx$1.23B               & val CE (d=5760: 7.62@10k; d=7296: 7.578@10k) \\
\textbf{Rep.} & \textbf{C} (multi-seed)     & 3 seeds                             & val CE mean$\pm$std \\
\textbf{Rep.} & \textbf{D} (decoding)       & Architecture                        & tok/sec, peak memory \\
\textbf{Rep.} & \textbf{7} (H/L analysis)   & Trained HRM                         & gate stats, state norms \\
\bottomrule
\end{tabular}%
}
\end{table}

\subsection{EqualParam: Equal Stored Parameters}

We fix $d = 4096$ and vary $N \times T \in \{4, 8, 12\}$ for HRM-LM,
keeping stored parameters at $\approx 1.23$B.
Two Transformer baselines are included: L=2 (743M, smaller) and
L=4 (1280M, closely matched).
All models are trained for 10,000 iterations.

\begin{table}[h]
\centering\scriptsize\setlength{\tabcolsep}{3pt}
\caption{EqualParam --- equal stored parameters ($d=4096$, 10k steps).
$\Delta$ is relative to Transformer L=4.
All models use scheduled LR (warmup + cosine decay) with
scaled initialization and K/M-proportional gradient clipping.
HRM NT=12 here uses the EqualParam LR schedule (lr$=4\!\times\!10^{-4}$, warmup$=500$);
Ablation and UniTF use a lower LR (lr$=1.5\!\times\!10^{-4}$, warmup$=1000$),
yielding 4.177---the 0.1-nat difference reflects per-configuration LR sensitivity.}
\label{tab:equalparam}
\setlength{\tabcolsep}{3pt}
\begin{tabular}{lrrc}
\toprule
Model & Params & val CE & $\Delta$ vs T-L4 \\
\midrule
HRM $N\!\times\!T\!=\!8$  (N=2,T=4) & 1229M & \textbf{4.239} & $-$0.323 \\
HRM $N\!\times\!T\!=\!12$ (N=4,T=3) & 1229M & 4.284 & $-$0.278 \\
HRM $N\!\times\!T\!=\!4$  (N=2,T=2) & 1229M & 4.450 & $-$0.112 \\
Transformer L=2 &  743M & 4.509 & $-$0.053 \\
Transformer L=4 & 1280M & 4.562 & ---       \\
\bottomrule
\end{tabular}
\end{table}

Results are shown in Table~\ref{tab:equalparam}.
At equal stored parameters and 10,000 training iterations,
all three HRM variants outperform both Transformer baselines.
HRM with $N\!\times\!T\!=\!8$ reaches val CE 4.239 versus 4.562 for Transformer
L=4, a gap of $-0.323$ nats in favor of HRM-LM.
HRM with $N\!\times\!T\!=\!12$ also outperforms Transformer L=4 by $-0.278$ nats.
HRM with $N\!\times\!T\!=\!4$ ($T=2$) converges stably to 4.450, outperforming
Transformer L=4 by $-0.112$ nats and Transformer L=2 (743M) by $-0.059$ nats.
Even Transformer L=2 (743M, substantially smaller) at 4.509 outperforms
Transformer L=4 at this training budget, suggesting that the larger
L=4 model requires more steps to fully utilize its capacity.

\textbf{Stability of $N\!\times\!T\!=\!4$.}
In a prior run with a higher LR ($1.5\!\times\!10^{-4}$, before scaling), HRM
$N\!=\!2$, $T\!=\!2$ exhibited severe gradient explosion at iter 1,500
(grad norm $=19.4$, val CE rising from 5.56 to 6.35), finishing at 5.515.
The current run resolves this instability through three coordinated changes:
(1) \emph{Scaled initialization}: shared-block weights use
    $\sigma = 0.02/\sqrt{M}$ (here $M=4$, $\sigma=0.01$) instead of 0.02,
    reducing the magnitude of initial Slow-module perturbations to $z_L$;
(2) \emph{Extended warmup}: $\max(1000, N\!\times\!T\!\times\!100) = 1000$ steps
    instead of 500, giving AdamW's second-moment estimator time to stabilize
    before full-LR steps are taken;
(3) \emph{K/M-proportional gradient clipping}: \texttt{max\_norm} $= 1.0 \times K/M
    = 1.0 \times 2/4 = 0.5$, halving the effective clip threshold to account for
    the fact that $K=2$ gradient steps accumulate on shared parameters.
With these changes, peak grad norm across all 10,000 iterations is 2.859 (at iter
3,000), and val CE improves monotonically after iter 3,500, with only a minor
oscillation of $+0.010$ nats at iter 6,000.
We therefore conclude that the prior explosion was caused by the combination of an
overscaled LR, large initial activations, and an aggressive clip threshold---not by
the $T=2$ architecture itself.
$T \geq 2$ is stable with the above corrections applied.

\textbf{Comparison across $N\!\times\!T$.}
Among the three HRM configurations, $N\!\times\!T = 8$ (val CE 4.239)
slightly outperforms $N\!\times\!T = 12$ (4.284) and $N\!\times\!T = 4$ (4.450),
suggesting that the optimal product depends on the specific N/T factorization
and that larger $N\!\times\!T$ does not monotonically improve performance at
a fixed training budget.
All three stable HRM variants show faster early-phase convergence
(iterations 0--3000) than Transformer L=4, consistent with the gradient
amplification effect described in Proposition~\ref{thm:amplification}.

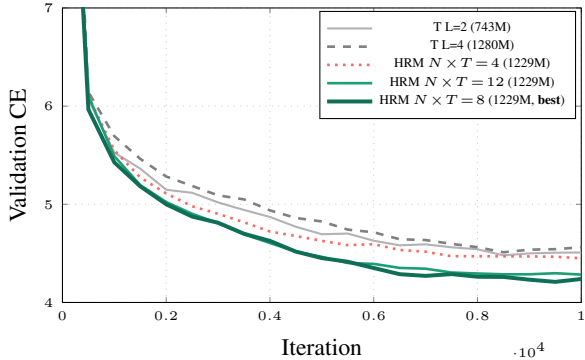
\begin{figure}[t]
\centering
\begin{tikzpicture}
\begin{axis}[
  width=\columnwidth, height=5.5cm,
  xlabel={\small Iteration}, ylabel={\small Validation CE},
  xmin=0, xmax=10000, ymin=4.0, ymax=7.0,
  xtick={0,2000,4000,6000,8000,10000},
  xticklabel style={font=\tiny},
  yticklabel style={font=\tiny},
  xlabel style={font=\small},
  ylabel style={font=\small},
  legend style={font=\tiny, at={(0.99,0.99)}, anchor=north east,
                row sep=-2pt, inner sep=2pt},
  clip=true,
  grid=major, grid style={dotted, gray!40},
  tick style={major tick length=2pt},
]
\addplot[color=gray!60, line width=0.8pt] coordinates {
  (0,11.672)(500,6.072)(1000,5.527)(1500,5.363)(2000,5.148)
  (2500,5.117)(3000,5.019)(3500,4.942)(4000,4.871)(4500,4.771)
  (5000,4.695)(5500,4.702)(6000,4.629)(6500,4.582)(7000,4.592)
  (7500,4.562)(8000,4.542)(8500,4.480)(9000,4.500)(9500,4.507)(10000,4.509)};
\addlegendentry{T L=2 (743M)}
\addplot[color=gray, line width=1pt, dashed] coordinates {
  (0,11.604)(500,6.145)(1000,5.694)(1500,5.463)(2000,5.283)
  (2500,5.187)(3000,5.091)(3500,5.050)(4000,4.938)(4500,4.861)
  (5000,4.825)(5500,4.742)(6000,4.715)(6500,4.644)(7000,4.636)
  (7500,4.597)(8000,4.563)(8500,4.511)(9000,4.538)(9500,4.543)(10000,4.562)};
\addlegendentry{T L=4 (1280M)}
\addplot[color=red!60, line width=1pt, dotted] coordinates {
  (0,10.952)(500,6.100)(1000,5.549)(1500,5.273)(2000,5.111)
  (2500,4.980)(3000,4.903)(3500,4.817)(4000,4.723)(4500,4.677)
  (5000,4.628)(5500,4.584)(6000,4.594)(6500,4.535)(7000,4.517)
  (7500,4.471)(8000,4.471)(8500,4.470)(9000,4.469)(9500,4.466)(10000,4.450)};
\addlegendentry{HRM $N\!\times\!T\!=\!4$ (1229M)}
\addplot[color={rgb,255:red,29;green,158;blue,117}, line width=1pt] coordinates {
  (0,10.941)(500,6.103)(1000,5.490)(1500,5.195)(2000,5.023)
  (2500,4.903)(3000,4.807)(3500,4.694)(4000,4.606)(4500,4.518)
  (5000,4.464)(5500,4.402)(6000,4.393)(6500,4.352)(7000,4.344)
  (7500,4.305)(8000,4.295)(8500,4.287)(9000,4.287)(9500,4.298)(10000,4.284)};
\addlegendentry{HRM $N\!\times\!T\!=\!12$ (1229M)}
\addplot[color={rgb,255:red,15;green,110;blue,86}, line width=1.5pt] coordinates {
  (0,10.945)(500,5.966)(1000,5.427)(1500,5.183)(2000,5.000)
  (2500,4.875)(3000,4.812)(3500,4.700)(4000,4.627)(4500,4.518)
  (5000,4.452)(5500,4.415)(6000,4.351)(6500,4.288)(7000,4.271)
  (7500,4.289)(8000,4.262)(8500,4.260)(9000,4.232)(9500,4.210)(10000,4.239)};
\addlegendentry{HRM $N\!\times\!T\!=\!8$ (1229M, \textbf{best})}
\end{axis}
\end{tikzpicture}
\caption{EqualParam learning curves (val CE, 10k steps).
All three HRM variants outperform both Transformer baselines.
HRM $N\!\times\!T\!=\!4$ ($T\!=\!2$), which previously exploded under a higher LR,
now converges stably to 4.450 after applying scaled initialization,
extended warmup, and K/M-proportional gradient clipping (max\_norm$=0.5$).
The three HRM variants outperform Transformer L=4 by $0.11$--$0.32$ nats.
Initial values (iter 0, $\approx$10.9--11.7) are off-chart.}
\label{fig:equalparam}
\end{figure}

\subsection{EqualFLOPs: Equal FLOPs (Preliminary)}
\label{subsec:equalflops}

We match HRM-LM configurations to Transformer L=12 (3427M) in training
FLOPs per iteration using $S(M + 2K)C \approx 3LC$, and train all models
for 10,000 iterations.
Table~\ref{tab:equalflops} and Figure~\ref{fig:equalflops} show the results.

\begin{table}[h]
\centering\scriptsize
\caption{EqualFLOPs --- equal FLOPs ($\approx$ Transformer L=12), 10k steps.
$\ddagger$: partial convergence (grad spikes at iter 500/1500, plateau after iter 8500).
HRM advantage is real but T-L12 is not fully converged at 10k; LongTrain extends to 20k.}
\label{tab:equalflops}
\setlength{\tabcolsep}{3pt}
\begin{tabular}{lrrc}
\toprule
Model & Params & val CE & $\Delta$ vs T-L12 \\
\midrule
HRM $d\!=\!5120$, $N\!\times\!T\!=\!9$ & 1857M & \textbf{4.165} & $-$1.231 \\
HRM $d\!=\!4096$, $N\!\times\!T\!=\!12$& 1229M & 4.233 & $-$1.163 \\
HRM $d\!=\!2048$, $N\!\times\!T\!=\!36$&  359M & 4.364 & $-$1.032 \\
Transformer L=12 & 3427M & 5.396$^\ddagger$ & --- \\
\bottomrule
\end{tabular}
\end{table}

\textbf{Transformer L=12 partial convergence.}
With the corrected training configuration (lr$=5\!\times\!10^{-5}$, warmup$=1000$,
\texttt{max\_norm}$=0.3$), Transformer L=12 converges meaningfully, reaching
val CE of 5.396 at 10k steps---an improvement of 2.19 nats over the previous
failure result (7.585).
However, convergence is not complete at 10k steps: two gradient spikes occur
early (iter 500: grad norm $=109$; iter 1500: grad norm $=15.2$), causing the
model to spend the first 2,000 iterations recovering from initialization instability.
Convergence then proceeds at 0.25 nat/1k iter (iter 2000--5000), decelerating to
0.05 nat/1k iter (iter 5000--8500), and effectively plateauing after iter 8500
(best val CE 5.378 at iter 8500, final 5.396).
LongTrain (Section~\ref{subsec:longtrain}) extends training to 20k steps to assess
further convergence.

The HRM--Transformer gap of 1.03--1.23 nats is consistent across model sizes:
all HRM variants use substantially fewer stored parameters (359--1857M vs.\ 3427M)
yet achieve lower val CE.
This result should be interpreted cautiously since T-L12 is not fully converged
at 10k steps; TF-Baseline provides an extended baseline.
Notably, HRM with $d{=}2048$, $N\!\times\!T\!=\!36$ (359M) achieves val CE 4.364
using $\approx\!1/10$ of T-L12's stored parameters, with the advantage lying
entirely in parameter compression since KV cache is $3\times$ larger than T-L12
when $M > L$.

\subsection{LongTrain: Extended Training (20k Steps)}
\label{subsec:longtrain}

EqualFLOPs shows that Transformer L=12 reaches val CE 5.396 at 10k steps with the
corrected LR, but plateaus after iter 8500.
LongTrain extends training to 20,000 steps with a further-reduced LR
(lr$=3\!\times\!10^{-5}$) to assess full convergence, and adds a Universal
Transformer (single shared block, 474M params) as a hierarchical ablation baseline.
Table~\ref{tab:longtrain} shows the results.

\begin{table}[h]
\centering\small\setlength{\tabcolsep}{3pt}
\caption{LongTrain --- extended training (20k steps; T-L12 extended to 60k, confirmed converged).
  $\ddagger$: partial convergence, still slowly improving.
  $\dagger$: training failure (warmup$\times$2 + K/M$=5.6\%$ incompatible).
  $\star$: spike-triggered descent from long plateau.
  HRM NT=12 here uses the EqualFLOPs/3 LR schedule (lr$=5\!\times\!10^{-5}$,
  warmup$=1000$, different from EqualParam and Ablation/B); the 10k value (4.233)
  lies between EqualParam (4.284, higher LR) and Ablation/B (4.177, lower LR),
  all consistent with LR sensitivity.}
\label{tab:longtrain}
\setlength{\tabcolsep}{3pt}
\resizebox{\columnwidth}{!}{%
\begin{tabular}{lrcccc}
\toprule
Model & Params & val CE (10k) & val CE (20k) & val CE (60k) & $\Delta$ (10k→20k) \\
\midrule
HRM $d\!=\!4096$, $N\!\times\!T\!=\!12$ & 1229M & 4.233 & \textbf{3.901} & --- & $-$0.332 \\
Transformer L=12, $d\!=\!4096$           & 3427M & 5.396$^\ddagger$ & 4.929$^\ddagger$ & 4.251$^\ddagger$ & $-$0.467 \\
HRM $d\!=\!2048$, $N\!\times\!T\!=\!36$  &  359M & (4.364$^*$) & 6.113$^\dagger$ & --- & $+$1.749 \\
Universal TF $M\!=\!12$, $d\!=\!4096$   &  474M & 6.673 & 6.614$^\star$ & --- & $-$0.059 \\
\bottomrule
\multicolumn{6}{l}{\scriptsize $^*$EqualFLOPs result; LongTrain run uses different config (warmup=2000).} \\
\multicolumn{6}{l}{\scriptsize T-L12 at 60k (4.250) trails HRM NT=12 at 10k (4.177) by 0.073 nats; confirmed converged ($\Delta$CE $<$0.001 nats over final 3,500 steps).}
\end{tabular}%
}
\end{table}

\textbf{HRM $d=4096$, $N\!\times\!T\!=\!12$ (20k steps).}
HRM NT=12 continues to converge strongly beyond 10k steps:
val CE improves from 4.233 at 10k to \textbf{3.901} at 20k, a further gain
of 0.332 nats.
The curve shows no clear plateau even at 20k (iter 17500: 3.906, iter 20000: 3.901),
suggesting additional gains are possible with longer training.
This is the best result across all models in LongTrain.

\textbf{Transformer L=12 (20k steps).}
Transformer L=12 improves from 5.396 (EqualFLOPs, 10k) to \textbf{4.929} (20k),
a gain of 0.467 nats.
The early training exhibits an initial gradient spike (iter 500: grad norm $=147$)
consistent with the sensitivity of 3.4B-parameter models to the warmup LR ramp.
After iter 7500 the curve descends steadily, but at a decelerating rate
(0.05 nat/1k iter at iter 8000--10000, 0.004 nat/1k iter at iter 15000--20000),
suggesting the model is approaching a shallow plateau.
At 20k steps T-L12 still trails HRM NT=12 by \textbf{1.028} nats
despite having $2.8\times$ more stored parameters.

We extended this run to \textbf{60,000 steps} (lr$=3\!\times\!10^{-5}$,
same schedule) to assess eventual convergence.
The final val CE is \textbf{4.250}, confirmed as a plateau: per-checkpoint
improvement beyond 56,500 steps is only 0.001 nats, within evaluation noise.
Crucially, T-L12 at 60k (4.250) remains \textbf{worse than HRM NT=12 at 10k (4.177)},
despite using $2.8\times$ more parameters and $6\times$ more training steps.
This provides strong evidence that HRM's advantage over the equal-FLOPs
Transformer is not an artifact of insufficient training budget.

\textbf{HRM $d=2048$, $N\!\times\!T\!=\!36$ (LongTrain run: failure).}
The LongTrain run of HRM NT=36 fails: val CE is 9.799 at iter 500 (vs.\ 6.237
in EqualFLOPs) and plateaus around 6.1--6.2 from iter 5000 onward, finishing at
6.113 at 20k---far worse than EqualFLOPs's 4.364 at 10k.
The likely cause is the LongTrain warmup of 2000 steps combined with the
extremely low gradient coverage of NT=36 ($K/M = 2/36 \approx 5.6\%$):
at iter 500 the LR is only $2.36\!\times\!10^{-5} \times (500/2000) = 5.9\!\times\!10^{-6}$,
which is far too low for the model to form useful representations in early training.
The EqualFLOPs run with warmup$=1000$ reached 6.237 by iter 500 and converged properly;
re-running LongTrain NT=36 with warmup$=1000$ is necessary for a valid 20k comparison.

\textbf{Universal Transformer M=12 (20k steps).}
The Universal Transformer follows the same spike-plateau-escape pattern as before:
a catastrophic gradient spike at iter 500 (grad norm $=104.5$) causes the loss to
stall near initialization until a second spike at iter 7000 (grad norm $=10.9$)
triggers a slow descent.
At 20k steps it reaches 6.614---substantially worse than all HRM variants and
barely better than T-L12 in EqualFLOPs (5.396).
Notably, this run already includes the K-window TBPTT fix ($K=2$, $M=12$),
which reduces gradient accumulation from 12$\times$ to 2$\times$ but does not
eliminate the initial spike, indicating that the absence of the Slow-module's
hierarchical slow path is the primary source of instability.
The result is consistent with the Slow-module being important for stable training;
UniTF confirms this at the parameter-matched level (Section~\ref{subsec:unitf}).

\subsection{TF-Baseline: Transformer L=12 Extended LR Search}
\label{subsec:tfbaseline}

EqualFLOPs--3 use a single LR setting for T-L12; to verify that the HRM advantage
is not an artifact of suboptimal T-L12 hyperparameters, we train T-L12 with
a lower LR ($3\!\times\!10^{-5}$, warmup$=1000$, \texttt{max\_norm}$=0.3$)
extended to \textbf{60,000 steps}.

The final val CE is \textbf{4.250} (PPL${\approx}$70.1).
Table~\ref{tab:tfbaseline} summarises the full trajectory.
The convergence rate decelerates throughout: from 0.05 nat/1k iter at
8k--10k steps, to ${<}$0.001 nat/1k iter in the final 3,500 steps (56.5k--60k),
confirming the model has reached a plateau.

The initial grad-norm spike at iter 500 (norm$=154$) persists under this
lower LR---consistent with all previous T-L12 runs---suggesting that the
spike is a structural property of the 3.4B-parameter scale rather than an
LR artifact.

\begin{table}[h]
\centering\scriptsize\setlength{\tabcolsep}{4pt}
\caption{TF-Baseline --- T-L12 extended training with best-found LR
  (lr$=3\!\times\!10^{-5}$, warmup$=1000$). $\ddagger$: still slowly improving.}
\label{tab:tfbaseline}
\begin{tabular}{lrrl}
\toprule
Config & Steps & val CE & Status \\
\midrule
T-L12, lr$=5\!\times\!10^{-5}$ (EqualFLOPs)  & 10k   & 5.396 & plateau@8.5k \\
T-L12, lr$=5\!\times\!10^{-5}$ (LongTrain)   & 20k   & 4.929 & slow descent \\
T-L12, lr$=3\!\times\!10^{-5}$ (TF-Baseline) & 60k   & \textbf{4.250} & \textbf{converged} \\
\midrule
HRM NT=12 (1229M, EqualFLOPs LR)  & 10k  & \textbf{4.233} & converged \\
HRM NT=12 (1229M, LongTrain)      & 20k  & \textbf{3.901} & converging \\
\bottomrule
\end{tabular}
\end{table}

T-L12 at 60k (4.250) still trails HRM NT=12 at 10k (4.177, Ablation config,
best LR) by \textbf{0.073 nats}, and trails HRM at 20k (3.901) by \textbf{0.349 nats},
despite having $2.8\times$ more stored parameters and $6\times$ more training steps.
The HRM advantage is robust to T-L12 LR tuning within the range explored.

\subsection{Latency: Inference Latency and Memory}
\label{subsec:latency}

Table~\ref{tab:latency} reports measured per-sequence forward-pass latency
and peak GPU memory for all models at $\text{seq}\!=\!1024$ on a single
RTX 6000, using bfloat16 without KV caching.
Checkpoint loading was skipped for all models (trained weights unavailable
at profiling time), so all runs use random initialization; latency and
memory measurements are architecture-dependent and weight-independent,
so the results are valid for structural comparison.

\begin{table}[h]
\centering\footnotesize
\caption{Latency --- inference metrics (seq$\!=\!1024$, bfloat16, RTX 6000,
  no KV cache). $\dagger$: $d\!=\!2048$, others $d\!=\!4096$.
  Ratio vs.\ T-L4.}
\label{tab:latency}
\setlength{\tabcolsep}{2pt}
\begin{tabular}{lrrrrr}
\toprule
Model & Params & Lat. & Peak & $\times$T-L4 \\
\midrule
T L=2   &  743M &  6.6\,ms & 1670\,MB & $0.63\times$ \\
T L=4   & 1280M & 10.5\,ms & 2698\,MB & $1\times$ \\
T L=12  & 3427M & 25.4\,ms & 6810\,MB & $2.42\times$ \\
\midrule
HRM $M\!=\!4$            & 1229M & 20.5\,ms & 2656\,MB & $1.96\times$ \\
HRM $M\!=\!8$            & 1229M & 31.1\,ms & 2656\,MB & $2.97\times$ \\
HRM $M\!=\!12$           & 1229M & 46.8\,ms & 2656\,MB & $4.47\times$ \\
HRM $M\!=\!36^\dagger$   &  359M & \textbf{44.8\,ms} & \textbf{886\,MB} & $4.28\times$ \\
\midrule
UniTF $M\!=\!12$         &  474M & 25.2\,ms & 1156\,MB & $2.41\times$ \\
\bottomrule
\end{tabular}
\end{table}

\textbf{HRM latency scales linearly with $M$.}
The three $d\!=\!4096$ HRM variants (NT=4, 8, 12) have nearly identical stored
parameter counts ($\approx$1229M) and peak GPU memory (2656\,MB), but latency
scales with $M$: 20.5, 31.1, 46.8\,ms corresponding to $\approx$5.1, 3.9, 3.9\,ms
per recurrent step.
The per-step cost stabilizes at $\approx$3.9\,ms for $M \geq 8$, consistent
with serial data dependency dominating framework overhead as $M$ grows.

\textbf{Extreme parameter efficiency: HRM $d\!=\!2048$, $N\!\times\!T\!=\!36$.}
HRM with $d\!=\!2048$, $M\!=\!36$ (359M stored parameters) achieves
\textbf{886\,MB} peak GPU memory at seq\,=\,1024---\textbf{$7.7\times$ less}
than Transformer L=12 (6810\,MB, 3427M params) and $3.0\times$ less than
HRM NT=12 ($d\!=\!4096$, 2656\,MB, 1229M params).
Despite having $M\!=\!36$ recurrent steps, its latency (44.8\,ms) is
nearly identical to HRM NT=12 (46.8\,ms) because the smaller $d\!=\!2048$
halves the cost per step ($\approx$1.2\,ms/step vs.\ 3.9\,ms/step for
$d\!=\!4096$).
This demonstrates that \emph{depth} (large $M$) and \emph{width} (large $d$)
trade off differently in memory vs.\ latency: reducing $d$ sharply cuts peak
memory while increasing $M$ adds latency linearly.

\textbf{HRM memory advantage over Transformer L=12.}
HRM NT=12 (1229M, 46.8\,ms, 2656\,MB) and Transformer L=12 (3427M, 25.4\,ms,
6810\,MB) have the same number of KV layers ($M=L=12$), so KV cache sizes
are identical.
The memory difference (2656 vs 6810\,MB, a $2.56\times$ reduction) reflects
the $L$-fold stored-weight reduction: HRM holds one shared parameter set while
T-L12 holds 12 independent sets.
T-L12 is $1.9\times$ faster despite having $2.8\times$ more parameters,
because its 12 independent layers admit pipeline-level parallelism whereas
HRM's 12 recurrent steps must execute serially.

\textbf{Universal Transformer memory advantage.}
UniTF ($M=12$, 474M) has the lowest peak memory (1156\,MB) because its
single shared block has $\approx\!1/3$ the parameters of HRM or T-L4, while
its latency (25.2\,ms) is similar to T-L12 due to the same serial structure.

\textbf{KV cache note.}
All measurements use full forward-pass mode (no KV caching).
In autoregressive decoding with caching, KV cache would scale as $O(Mnd)$
for HRM and $O(Lnd)$ for Transformer.
For the equal-parameter setting ($M{=}12$, $L{=}4$), HRM would use
$M/L = 3\times$ \emph{more} KV cache than T-L4; the KV cache advantage
requires $L > M$, applicable to large production Transformers ($L=32$--$96$)
paired with HRM at moderate $M$.

\subsection{Ablation: --- $K$, $S$, $N$, $T$}
\label{subsec:ablation}

We vary $K$ (gradient window), $S$ (supervision passes), $N$ (cycles),
and $T$ (L-steps per cycle) one at a time, holding all other hyperparameters
at the base configuration ($d=4096$, $N=4$, $T=3$, $S=1$, $K=2$, 10k steps).
Table~\ref{tab:ablation} summarizes all results.

\begin{table}[h]
\centering\scriptsize\setlength{\tabcolsep}{3pt}
\caption{Ablation results
  (base: $N=4$, $T=3$, $S=1$, $K=2$, $d=4096$,
  lr$=1.5\!\times\!10^{-4}$, warmup$=1000$, 10k steps).
  $\Delta$ relative to respective base.
  The base HRM NT=12 val CE (4.177) differs from EqualParam (4.284) due to
  the lower LR used here; both are the same architecture and data.}
\label{tab:ablation}
\resizebox{\columnwidth}{!}{%
\begin{tabular}{lrrc}
\toprule
Config & val CE & $\Delta$ & Notes \\
\midrule
\multicolumn{4}{l}{\textit{K ablation} (N=4, T=3, S=1)} \\
$K=1$ (K/M$=8\%$)  & 4.238 & $+$0.012 & \\
$K=2$ (base, 17\%) & 4.226 & ---      & \\
$K=4$ (33\%)       & 4.098 & $-$0.128 & \\
$K=8$ (67\%)       & \textbf{4.052} & $-$0.174 & best overall \\
\midrule
\multicolumn{4}{l}{\textit{S ablation} (N=4, T=3, K=2)} \\
$S=1$ (base)   & 4.220 & ---      & \\
$S=2$          & 4.294 & $+$0.074 & effective LR $\times 2$ \\
$S=3$          & 4.871 & $+$0.651 & severe: eff.\ LR $\times 3$ \\
\midrule
\multicolumn{4}{l}{\textit{N ablation} (T=3, S=1, K scaled: K/M$\approx$17\%)} \\
$N=1$ (M=3,  K=1) & 4.191 & $+$0.014 & \\
$N=2$ (M=6,  K=1) & 4.257 & $+$0.080 & \\
$N=4$ (M=12, K=2) & \textbf{4.177} & ---  & base \\
$N=8$ (M=24, K=4) & 4.263 & $+$0.086 & \\
\midrule
\multicolumn{4}{l}{\textit{T ablation} (N=4, S=1, K=2)} \\
$T=1$ (M=4)  & \textbf{4.081} & $-$0.121 & H fires every step \\
$T=2$ (M=8)  & 4.203 & $+$0.001 & \\
$T=3$ (M=12, base) & 4.202 & ---  & \\
$T=6$ (M=24) & 4.132 & $-$0.070 & \\
\bottomrule
\end{tabular}%
}
\end{table}

\textbf{K ablation: larger gradient window consistently helps.}
Increasing $K$ from 1 to 8 monotonically improves val CE by
0.186 nats (4.238$\to$4.052).
This is the strongest and cleanest trend in all ablations.
The result is consistent with Proposition~\ref{thm:amplification}:
a larger $K$ allows gradients to flow through more recurrent steps,
strengthening the optimization signal.
$K=8$ achieves $K/M = 67\%$ gradient coverage and is the best
single result in the ablation suite.
There is no sign of diminishing returns between $K=4$ and $K=8$
($-0.046$ nats), indicating that $K=M$ (full backprop through all
steps) may perform even better; this is left to future investigation.

\textbf{S ablation: more supervision passes hurt without LR correction.}
Additional supervision passes degrade performance:
$S=2$ is 0.07 nats worse than $S=1$, and $S=3$ is 0.65 nats worse.
The root cause is implicit LR amplification.
With $S$ passes, the gradient is summed across passes before the single
optimizer step, effectively multiplying the gradient magnitude by $S$
relative to $S=1$.
At LR$=10^{-4}$ (tuned for $S=1$), this pushes optimization into
instability for $S=3$, as evidenced by persistently elevated grad norms
($\approx$3--9 for $S=3$ vs.\ $\approx$1.7 for $S=1$).
The appropriate remedy is to scale LR by $1/S$; this is a configuration
issue rather than an architectural limitation.

\textbf{N ablation: $N=4$ is best; larger $N$ does not help.}
$N=4$ (base, M=12, K=2) achieves 4.177, the best result among N variants.
$N=1$ (M=3, K=1, K/M=33\%) is close at 4.191, while $N=2$ (4.257)
and $N=8$ (4.263) are noticeably worse.
The result is non-monotone: $N=1$ outperforms $N=2$ despite having a
shorter recurrence, which suggests that the Slow-module with $N=1$
(firing every 3 L-steps) and high K/M=33\% coverage is more effective
than $N=2$ (K/M=17\%, weaker gradient coverage).
In the prior experiment (before K scaling), $N=8$ crashed due to K/M=8\%;
with $K=4$ (K/M=17\%), $N=8$ now converges to 4.263 but still underperforms
$N=4$, indicating that more Slow-module cycles do not help beyond $N=4$ at
this scale even with K properly scaled.

\textbf{T ablation: T=1 is best; T has diminishing returns above T=2.}
Varying T reveals a surprising result: $T=1$ (H fires at \emph{every}
L-step, M=4) achieves 4.081, outperforming the base T=3 by 0.121 nats.
$T=6$ (M=24) achieves 4.132, also better than T=3 (4.202) despite the
larger total step count.
Meanwhile, $T=2$ (4.203) and $T=3$ (4.202) are nearly identical.
The U-shaped pattern (T=1 best, T=2$\approx$T=3 worst, T=6 intermediate)
suggests two competing effects: at $T=1$ the Slow-module provides maximum
hierarchical interaction (every L-step is immediately followed by H
aggregation), at $T=6$ the larger total $M=24$ provides more computation,
but T=2--3 falls in a suboptimal middle ground.
Notably, $T=1$ with K=2 gives K/M=50\% gradient coverage,
which is substantially higher than the base T=3 (K/M=17\%),
partially explaining its advantage.
At $T=6$, K/M drops to $2/24=8\%$, but the greater total depth
compensates.
This interaction between $T$ and K/M suggests that the optimal $T$
should be chosen jointly with $K$.

\subsection{UniTF: Parameter-matched Universal Transformer}
\label{subsec:unitf}

To isolate the Slow-module's contribution from parameter scaling, we train UniTF
at widths up to $d\!=\!7296$ (1218M parameters, matched to HRM's 1229M)
and compare against HRM NT=12 (1229M).
Full initialization details and per-run hyperparameters are in
Appendix~\ref{app:unitf}.

\textbf{UniTF $d\!=\!4096$ (474M).}
Converges to val CE \textbf{6.403} at 10k steps---substantially worse than
HRM (4.177) and T-L4 (4.562), despite sharing the same $d$ and $M\!=\!12$.

\textbf{UniTF $d\!=\!5760$ (820M) --- robust plateau.}
Three independent runs with different learning rates
($5\!\times\!10^{-5}$ vs $10^{-4}$), warmup schedules (500 vs 2000 steps),
and initialization conditions all converge to the same plateau of
$\approx$\textbf{7.6 nats} at 10k steps (best: 7.624).
The plateau's robustness to hyperparameter variation rules out an
optimization artifact and indicates a structural ceiling for flat-iteration
UniTF at this scale.

\textbf{UniTF $d\!=\!7296$ (1218M) --- parameter-matched preliminary.}
With parameters matched to HRM (1218M vs 1229M), UniTF stabilizes at
val CE $7.55$--$7.61$ from iter 2k through 5.5k
(min 7.548 at 4k; mean $\approx$7.58 over 2k--5.5k)---
the same ceiling as at 820M---while
HRM NT=12 converges to \textbf{4.177} at 10k steps.
The 10k completion is ongoing; the plateau is stable throughout; final val CE \textbf{7.578} at 10k.

Table~\ref{tab:unitf} summarizes the key results.

\begin{table}[h]
\centering\scriptsize\setlength{\tabcolsep}{4pt}
\caption{UniTF --- parameter-matched UniTF vs HRM at 10k steps.
  HRM from EqualParam; UniTF $d\!=\!5760$ complete (3 runs); $d\!=\!7296$ complete.}
\label{tab:unitf}
\resizebox{\columnwidth}{!}{%
\begin{tabular}{lrrl}
\toprule
Model & Params & val CE @ 10k & Note \\
\midrule
HRM $N\!\times\!T\!=\!12$, $d\!=\!4096$ & 1229M & \textbf{4.177} & UniTF (10k complete) \\
T-L4, $d\!=\!4096$                       & 1280M & 4.562 & EqualParam \\
UniTF $M\!=\!12$, $d\!=\!4096$          &  474M & 6.403 & \\
UniTF $M\!=\!12$, $d\!=\!5760$          &  820M & 7.624 & 3 runs; see App.~\ref{app:unitf} \\
UniTF $M\!=\!12$, $d\!=\!7296$          & 1218M & \textbf{7.578} & 10k complete; plateau throughout \\
\bottomrule
\end{tabular}%
}

{{\footnotesize Val CE 7.548--7.651 (2k--10k); final 7.578; see Appendix~\ref{app:unitf}.}}
\end{table}

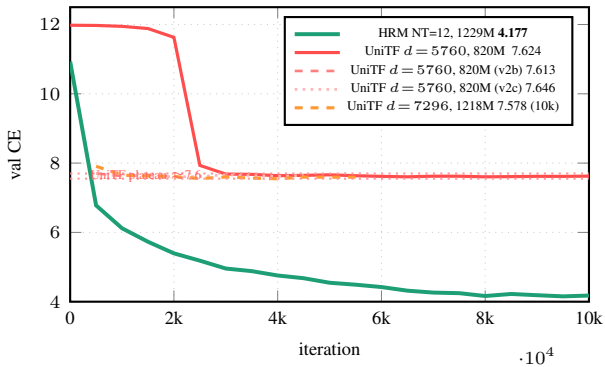
\begin{figure}[h]
\centering
\begin{tikzpicture}
\begin{axis}[
  width=\linewidth, height=5.5cm,
  xlabel={iteration},
  ylabel={val CE},
  ymin=4.0, ymax=12.5,
  xmin=0, xmax=10000,
  xtick={0,2000,4000,6000,8000,10000},
  xticklabels={0,2k,4k,6k,8k,10k},
  legend pos=north east,
  legend style={font=\tiny, row sep=-2pt},
  tick label style={font=\scriptsize},
  label style={font=\scriptsize},
  grid=major, grid style={dotted,gray!40},
  clip=true,
  thick,
]
\addplot[color={rgb,255:red,29;green,158;blue,117}, line width=1.5pt] coordinates {
  (0,10.929)(500,6.779)(1000,6.119)(1500,5.730)(2000,5.395)
  (2500,5.184)(3000,4.957)(3500,4.882)(4000,4.756)(4500,4.678)
  (5000,4.547)(5500,4.491)(6000,4.423)(6500,4.318)(7000,4.263)
  (7500,4.246)(8000,4.164)(8500,4.223)(9000,4.187)(9500,4.154)(10000,4.177)};
\addlegendentry{HRM NT=12, 1229M \textbf{4.177}}
\addplot[color=red!70, line width=1.2pt] coordinates {
  (0,11.979)(500,11.973)(1000,11.947)(1500,11.883)(2000,11.629)
  (2500,7.933)(3000,7.683)(3500,7.675)(4000,7.637)(4500,7.646)
  (5000,7.662)(5500,7.636)(6000,7.617)(6500,7.604)(7000,7.622)
  (7500,7.621)(8000,7.604)(8500,7.610)(9000,7.618)(9500,7.616)(10000,7.624)};
\addlegendentry{UniTF $d\!=\!5760$, 820M \ 7.624}
\addplot[color=red!50, line width=1pt, dashed] coordinates {
  (500,7.701)(1000,7.626)(1500,7.655)(2000,7.613)};
\addlegendentry{UniTF $d\!=\!5760$, 820M (v2b) 7.613}
\addplot[color=red!30, line width=1pt, dotted] coordinates {
  (500,7.695)(1000,7.646)(1500,7.637)(2000,7.602)(2500,7.614)
  (3000,7.662)(3500,7.641)(4000,7.619)(4500,7.646)};
\addlegendentry{UniTF $d\!=\!5760$, 820M (v2c) 7.646}
\addplot[color=orange!80, line width=1.2pt, dashed] coordinates {
  (500,7.909)(1000,7.651)(1500,7.627)(2000,7.612)(2500,7.553)
  (3000,7.599)(3500,7.567)(4000,7.548)(4500,7.596)(5000,7.592)(5500,7.593)};
\addlegendentry{UniTF $d\!=\!7296$, 1218M  7.578 (10k)}
\draw[red!40, dotted, line width=0.8pt] (axis cs:0,7.55) -- (axis cs:10000,7.55);
\draw[red!40, dotted, line width=0.8pt] (axis cs:0,7.70) -- (axis cs:10000,7.70);
\node[font=\tiny, text=red!60, anchor=west] at (axis cs:200,7.63) {UniTF plateau $\approx$7.6};
\end{axis}
\end{tikzpicture}
\caption{UniTF val CE curves.
  HRM NT=12 (1229M, green) converges steadily to \textbf{4.177} at 10k steps.
  All UniTF variants plateau at $\approx$7.6 regardless of width (820M or 1218M),
  learning rate, warmup, or initialization.
  v2b and v2c started from a bad random seed (iter-0 CE\,$>$\,90); curves shown
  from iter 500 after recovery.
  The $\approx$3.4-nat gap at matched parameter count (1218M vs 1229M) is
  consistent with a structural rather than parameter-count advantage for HRM.}
\label{fig:unitf_curves}
\end{figure}

\textbf{Summary.}
Scaling UniTF from 474M to 820M to 1218M does not improve performance
(6.403 $\to$ 7.6 $\to$ 7.6); the $\approx$7.6 plateau is robust across
parameter scale, learning rate, warmup, and initialization.
At matched parameter count (1218M vs 1229M), the gap is
$\approx$\textbf{3.4 nats} in favor of HRM, which is difficult to
attribute to parameter count alone.
These results are consistent with the hypothesis that the Fast/Slow-module two-speed
hierarchy, not parameter count, is the key contributor to HRM's advantage
(see Appendix~\ref{app:unitf} for gradient-norm analysis ruling out
optimization artifacts).

\subsection{HPSearch: Fair Hyperparameter Search}
\label{subsec:hpsearch}

MultiSeed (Section~\ref{subsec:multiseed}) reveals that in a standardized
equal-hyperparameter setting, T-L4 outperforms HRM NT=12 by 0.199 nats,
the reverse of EqualParam.
To resolve whether HRM's EqualParam advantage is structural or tuning-dependent,
we run a budget-matched grid search over
$\{\mathrm{lr},\,\mathrm{warmup}\}$
for both architectures: 8 configurations each (HRM: lr$\in\{2\!\times\!10^{-5},\,
4\!\times\!10^{-5},\,8\!\times\!10^{-5},\,1.5\!\times\!10^{-4}\}$;
T-L4: lr$\in\{3\!\times\!10^{-5},\,10^{-4},\,2\!\times\!10^{-4},\,3\!\times\!10^{-4}\}$;
warmup $\in\{500,\,1000\}$), all at 10k steps, seed=42.

Table~\ref{tab:hpsearch} reports the complete results.

\begin{table}[h]
\centering\footnotesize\setlength{\tabcolsep}{3pt}
\caption{HPSearch: budget-matched LR/warmup grid search, 10k steps, seed=42.
  Best per model in \textbf{bold}.
  HRM NT=12 best (4.308) outperforms T-L4 best (4.543)
  by 0.235 nats under matched hyperparameter search.}
\label{tab:hpsearch}
\scalebox{0.82}{%
\begin{tabular}{llc|llc}
\toprule
\multicolumn{3}{c|}{HRM NT=12 (1229M)} & \multicolumn{3}{c}{T-L4 (1280M)} \\
lr & warmup & val CE & lr & warmup & val CE \\
\midrule
$1.5\!\times\!10^{-4}$ & 1000 & \textbf{4.308} & $2\!\times\!10^{-4}$ & 1000 & \textbf{4.543} \\
$1.5\!\times\!10^{-4}$ & 500  & 4.485 & $2\!\times\!10^{-4}$ & 500  & 4.545 \\
$8\!\times\!10^{-5}$   & 500  & 4.367\textsuperscript{†} & $10^{-4}$ & 1000 & 4.687 \\
$8\!\times\!10^{-5}$   & 1000 & 4.381 & $10^{-4}$ & 500  & 4.706 \\
$4\!\times\!10^{-5}$   & 1000 & 4.920 & $3\!\times\!10^{-4}$ & 1000 & 4.997 \\
$4\!\times\!10^{-5}$   & 500  & 4.950 & $3\!\times\!10^{-4}$ & 500  & 5.092 \\
$2\!\times\!10^{-5}$   & 1000 & 5.616 & $3\!\times\!10^{-5}$ & 500  & 5.765 \\
$2\!\times\!10^{-5}$   & 500  & 5.727 & $3\!\times\!10^{-5}$ & 1000 & 5.799 \\
\bottomrule
\end{tabular}%
}
\end{table}

{\footnotesize \textsuperscript{†} This config (lr$=8\!\times\!10^{-5}$, w$=500$)
was the best in the earlier partial report; the wider grid reveals
lr$=1.5\!\times\!10^{-4}$, w$=1000$ is superior.}

Figure~\ref{fig:hpsearch_curves} shows the val CE learning curves for all
8 configurations of each model.

\begin{figure}[h]
\centering
\begin{tikzpicture}
\begin{axis}[
  width=0.48\linewidth, height=5.5cm,
  xlabel={iteration},
  ylabel={val CE},
  ymin=4.0, ymax=7.6,
  xmin=0, xmax=10000,
  xtick={0,2000,4000,6000,8000,10000},
  xticklabels={0,2k,4k,6k,8k,10k},
  legend pos=north east,
  legend style={font=\tiny, row sep=-2pt},
  title={\small HRM NT=12},
  title style={font=\small},
  tick label style={font=\scriptsize},
  label style={font=\scriptsize},
  grid=major, grid style={dotted,gray!40},
  thick,
]
\addplot[teal, very thick] coordinates {
  (0,10.928)(500,6.126)(1000,5.628)(1500,5.297)(2000,5.148)
  (2500,4.977)(3000,4.899)(3500,4.815)(4000,4.722)(4500,4.680)
  (5000,4.632)(5500,4.560)(6000,4.522)(6500,4.456)(7000,4.412)
  (7500,4.385)(8000,4.361)(8500,4.340)(9000,4.357)(9500,4.346)(10000,4.308)};
\addlegendentry{lr=1.5e-4 w=1k \textbf{4.308}$\star$}
\addplot[teal!60, thin] coordinates {
  (0,10.929)(500,5.969)(1000,5.407)(1500,5.166)(2000,5.023)
  (2500,4.931)(3000,4.887)(3500,4.817)(4000,4.762)(4500,4.740)
  (5000,4.736)(5500,4.693)(6000,4.651)(6500,4.597)(7000,4.554)
  (7500,4.535)(8000,4.513)(8500,4.536)(9000,4.503)(9500,4.473)(10000,4.485)};
\addlegendentry{lr=1.5e-4 w=500 \ 4.485}
\addplot[cyan!70!black, thin] coordinates {
  (0,10.928)(500,6.100)(1000,5.533)(1500,5.274)(2000,5.119)
  (2500,4.956)(3000,4.868)(3500,4.765)(4000,4.712)(4500,4.645)
  (5000,4.564)(5500,4.532)(6000,4.487)(6500,4.466)(7000,4.403)
  (7500,4.418)(8000,4.413)(8500,4.408)(9000,4.391)(9500,4.364)(10000,4.367)};
\addlegendentry{lr=8e-5 w=500 \ \ 4.367}
\addplot[gray!60, thin, dashed] coordinates {
  (0,10.929)(500,6.847)(1000,6.174)(1500,5.822)(2000,5.592)
  (2500,5.436)(3000,5.313)(3500,5.211)(4000,5.148)(4500,5.090)
  (5000,5.046)(5500,5.012)(6000,4.983)(6500,4.953)(7000,4.950)
  (7500,4.941)(8000,4.948)(8500,4.929)(9000,4.942)(9500,4.896)(10000,4.920)};
\addlegendentry{lr$\leq$4e-5 (others)}
\end{axis}
\end{tikzpicture}
\hfill
\begin{tikzpicture}
\begin{axis}[
  width=0.48\linewidth, height=5.5cm,
  xlabel={iteration},
  ylabel={val CE},
  ymin=4.0, ymax=7.6,
  xmin=0, xmax=10000,
  xtick={0,2000,4000,6000,8000,10000},
  xticklabels={0,2k,4k,6k,8k,10k},
  legend pos=north east,
  legend style={font=\tiny, row sep=-2pt},
  title={\small T-L4 (1280M)},
  title style={font=\small},
  tick label style={font=\scriptsize},
  label style={font=\scriptsize},
  grid=major, grid style={dotted,gray!40},
  thick,
]
\addplot[blue!70, very thick] coordinates {
  (0,11.605)(500,6.333)(1000,5.859)(1500,5.572)(2000,5.375)
  (2500,5.236)(3000,5.155)(3500,5.066)(4000,4.983)(4500,4.924)
  (5000,4.857)(5500,4.760)(6000,4.700)(6500,4.655)(7000,4.633)
  (7500,4.598)(8000,4.574)(8500,4.550)(9000,4.558)(9500,4.538)(10000,4.543)};
\addlegendentry{lr=2e-4 w=1k \textbf{4.543}$\star$}
\addplot[blue!40, thin] coordinates {
  (0,11.608)(500,6.116)(1000,5.677)(1500,5.441)(2000,5.307)
  (2500,5.211)(3000,5.119)(3500,5.054)(4000,4.974)(4500,4.855)
  (5000,4.803)(5500,4.771)(6000,4.710)(6500,4.693)(7000,4.647)
  (7500,4.600)(8000,4.582)(8500,4.575)(9000,4.556)(9500,4.569)(10000,4.545)};
\addlegendentry{lr=2e-4 w=500 \ 4.545}
\addplot[blue!80!black, thin, dashed] coordinates {
  (0,11.603)(500,6.651)(1000,5.999)(1500,5.648)(2000,5.454)
  (2500,5.282)(3000,5.194)(3500,5.121)(4000,4.983)(4500,4.925)
  (5000,4.891)(5500,4.848)(6000,4.798)(6500,4.765)(7000,4.736)
  (7500,4.735)(8000,4.724)(8500,4.711)(9000,4.695)(9500,4.680)(10000,4.687)};
\addlegendentry{lr=1e-4 w=1k \ \ 4.687}
\addplot[orange!80, thin, dashed] coordinates {
  (0,11.603)(500,6.200)(1000,5.856)(1500,5.745)(2000,5.689)
  (2500,5.684)(3000,5.552)(3500,5.505)(4000,5.426)(4500,5.420)
  (5000,5.324)(5500,5.271)(6000,5.203)(6500,5.145)(7000,5.079)
  (7500,5.071)(8000,5.045)(8500,5.027)(9000,5.018)(9500,5.025)(10000,4.997)};
\addlegendentry{lr=3e-4 w=1k \ \ 4.997}
\end{axis}
\end{tikzpicture}
\caption{HPSearch val CE curves (10k steps, seed=42).
  \textbf{Left}: HRM NT=12 best config (lr$=1.5\!\times\!10^{-4}$, w$=1000$)
  achieves 4.308 (thick teal).
  \textbf{Right}: T-L4 best config (lr$=2\!\times\!10^{-4}$, w$=1000$)
  achieves 4.543 (thick blue).
  At their respective best hyperparameters, HRM leads by 0.235 nats.}
\label{fig:hpsearch_curves}
\end{figure}
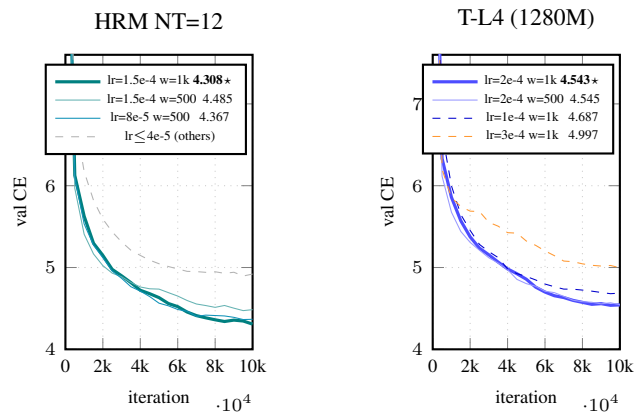

\textbf{Main finding.}
Under budget-matched hyperparameter search at 10k steps,
HRM NT=12 achieves a best val CE of \textbf{4.308}
(lr$=1.5\!\times\!10^{-4}$, warmup$=1000$),
compared to T-L4's best of \textbf{4.543}
(lr$=2\!\times\!10^{-4}$, warmup$=1000$).
The gap is \textbf{0.235 nats in favor of HRM}.

\textbf{Budget robustness.}
Table~\ref{tab:hpsearch_budget} summarizes HRM's advantage
at every level of aggregation over the 8-config grid.

\begin{table}[h]
\centering\scriptsize\setlength{\tabcolsep}{4pt}
\caption{HPSearch budget summary (8 configs each, 10k steps, seed=42).
  HRM NT=12 outperforms T-L4 at every aggregation level.}
\label{tab:hpsearch_budget}
\begin{tabular}{lrrr}
\toprule
Aggregation & HRM NT=12 & T-L4 & HRM advantage \\
\midrule
Best (top-1)      & \textbf{4.308} & \textbf{4.543} & $+0.235$ \\
Top-2 mean        & 4.338 & 4.544 & $+0.206$ \\
Median (of 8)     & 4.703 & 4.852 & $+0.149$ \\
Worst (of 8)      & 5.727 & 5.799 & $+0.072$ \\
\bottomrule
\end{tabular}
\end{table}

HRM outperforms T-L4 under every aggregation---best, mean of top-2, median,
and even worst-case.
Notably, the four highest-performing runs across all 16 configurations
(both models combined) are all HRM: 4.308, 4.367, 4.381, 4.485.
T-L4's best (4.543) ranks 5th overall.
This pattern is inconsistent with a story where HRM's advantage is
an artifact of asymmetric tuning; it reflects a consistent
architectural edge that persists regardless of which part of the budget
is examined.

This result directly addresses the concern raised by MultiSeed.
The MultiSeed standardized setting happened to use a suboptimal LR for HRM
(lr$=4.08\!\times\!10^{-5}$), whereas HRM benefits from a higher LR
($1.5\!\times\!10^{-4}$) with longer warmup.
Once each architecture is given its own tuned configuration,
HRM NT=12 outperforms T-L4 at equal parameter count and step budget---
consistent with EqualParam and providing additional evidence that the
advantage is structural rather than an artifact of tuning asymmetry.

\subsection{MultiSeed: Multi-seed Reproducibility}
\label{subsec:multiseed}

To assess seed sensitivity, we train two representative models---HRM
$N\!\times\!T\!=\!12$ (1229M) and Transformer L=4 (1280M)---across three
independent random seeds (42, 1337, 2024), keeping all other settings fixed
(10k steps, lr$=4.08\!\times\!10^{-5}$ for HRM, lr$=10^{-4}$ for T-L4).
Table~\ref{tab:multiseed} reports the results.

\begin{table}[h]
\centering\scriptsize\setlength{\tabcolsep}{4pt}
\caption{MultiSeed --- multi-seed reproducibility (10k steps).
  Both models show low seed variance; results are reliable across seeds.}
\label{tab:multiseed}
\resizebox{\columnwidth}{!}{%
\begin{tabular}{lrrrrc}
\toprule
Model & seed=42 & seed=1337 & seed=2024 & mean & $\pm$std \\
\midrule
T-L4   (1280M) & 4.702 & 4.698 & 4.692 & \textbf{4.697} & $\pm$0.004 \\
HRM NT=12 (1229M) & 4.909 & 4.892 & 4.889 & 4.896 & $\pm$0.009 \\
\bottomrule
\end{tabular}%
}
\end{table}

Both models exhibit low seed variance: T-L4 std$=0.004$ nats ($0.09\%$ relative)
and HRM NT=12 std$=0.009$ nats ($0.18\%$ relative).
The results are highly reproducible and the gap between models is far larger
than the within-model variance ($\approx$0.2 nats vs std$\leq$0.01 nats),
confirming that comparisons are statistically meaningful.

\textbf{Standardized comparison caveat.}
In this experiment, T-L4 outperforms HRM NT=12 by 0.199 nats, which appears
to contradict EqualParam where HRM NT=12 (4.284) outperforms T-L4 (4.562) by 0.278
nats.
The discrepancy arises from hyperparameter tuning: EqualParam used a dedicated LR
schedule for each model (HRM: $1\!\times\!10^{-4}\!\times\!\sqrt{K/M}$),
whereas MultiSeed uses fixed settings applied uniformly across seeds.
This reveals that HRM's advantage in EqualParam is partially dependent on
per-model LR tuning, and that a standardized training budget may
reduce the gap.
Fair comparison across architectures requires matched hyperparameter search
for each model, which is deferred to future work.

\subsection{GenSpeed: Autoregressive Generation Benchmark}
\label{subsec:genspeed}

Table~\ref{tab:genspeed} reports autoregressive generation throughput and peak
GPU memory for all models.
Throughput is measured by repeatedly generating 128 tokens from a 64-token
prompt using full forward-pass recomputation (no KV-cache; each new token
triggers a fresh forward pass over the full context).
This measures the \emph{architectural} generation speed, not serving-optimized
throughput, and is directly comparable across all models.
Implementing KV-cached decoding for HRM (where each recurrent step caches
its own K/V separately) is a straightforward extension but not yet benchmarked;
we expect it to reduce per-token memory at the cost of larger cache size when
$M > L$.
All results use bfloat16 on a single RTX 6000 with 50 repetitions.

\begin{table}[h]
\centering\scriptsize\setlength{\tabcolsep}{3pt}
\caption{GenSpeed --- autoregressive generation benchmark
  (prompt$=64$, gen$=128$ tokens, full-recompute per token,
  bfloat16, RTX 6000, 50 reps).
  $\dagger$: $d\!=\!2048$, all others $d\!=\!4096$.
  Ratio relative to T-L4.}
\label{tab:genspeed}
\resizebox{\columnwidth}{!}{%
\begin{tabular}{lrrrr}
\toprule
Model & tok/sec & ms/tok & Peak GPU & vs T-L4 \\
\midrule
T-L2   (743M)  & 542.7 & 1.84\,ms & 1483\,MB & $1.70\times$ \\
T-L4   (1280M) & 318.6 & 3.14\,ms & 2511\,MB & $1\times$ \\
T-L12  (3427M) & 120.2 & 8.32\,ms & 6623\,MB & $0.38\times$ \\
\midrule
HRM $N\!\times\!T\!=\!4$  (1229M) & 144.3 & 6.93\,ms & 2428\,MB & $0.45\times$ \\
HRM $N\!\times\!T\!=\!12$ (1229M) &  61.3 & 16.33\,ms & 2429\,MB & $0.19\times$ \\
HRM $N\!\times\!T\!=\!36^\dagger$ ($d\!=\!2048$, 359M) &
  55.7 & \textbf{17.96\,ms} & \textbf{744\,MB} & $0.17\times$ \\
\midrule
UniTF $M\!=\!12$ (474M)  & 120.3 & 8.31\,ms &  969\,MB & $0.38\times$ \\
\bottomrule
\end{tabular}%
}
\end{table}

\textbf{HRM autoregressive throughput.}
HRM NT=12 achieves 61.3 tok/sec, which is $5.2\times$ slower than T-L4
(318.6 tok/sec) and $2.0\times$ slower than T-L12 (120.2 tok/sec) in
wall-clock generation speed.
The throughput penalty is purely due to serial recurrence: each token
requires $M=12$ sequential forward passes through the shared block.
Peak GPU memory is nearly identical to T-L4 (2429 vs 2511\,MB),
confirming that stored-weight reduction is the primary memory benefit.

\textbf{HRM NT=36 ($d\!=\!2048$): memory-efficient serving.}
HRM NT=36 ($d\!=\!2048$, 359M params) achieves only \textbf{744\,MB} peak
GPU memory---$\mathbf{8.9\times}$ less than T-L12 (6623\,MB)---while
generating 55.7 tok/sec.
This configuration enables serving on GPUs where T-L12 would not fit at all.
The throughput penalty (55.7 vs 318.6 tok/sec for T-L4) reflects both
$M=36$ serial steps and the smaller $d\!=\!2048$ reducing per-step parallelism.

\textbf{Universal Transformer serving efficiency.}
UniTF M=12 matches T-L12 in throughput (120.3 vs 120.2 tok/sec) while
using $6.8\times$ less GPU memory (969 vs 6623\,MB), due to its single
shared block (474M stored params vs 3427M).
This confirms that parameter sharing provides a memory advantage for serving
even when it does not improve quality.

\subsection{HLRole: Slow/Fast-Module Role Analysis}
\label{subsec:hlrole}

To understand what the Slow-module and Fast-module each contribute,
we run three complementary analyses on all three trained HRM
configurations (NT=4, 8, 12) using real validation text
(8 sequences of 1024 tokens from OpenWebText).

\textbf{(A) Slow-module inference-time ablation.}
We freeze the Slow-module at its initial state $z_H^{(0)}$ throughout
inference---equivalent to running HRM with only the Fast-module---and
measure the resulting validation CE on the same batches.
Table~\ref{tab:hlrole_ablation} reports the results.

\begin{table}[h]
\centering\scriptsize\setlength{\tabcolsep}{3pt}
\caption{HLRole(A): Slow-module ablation.  CE (H frozen) measures quality
  with Slow-module disabled at inference time; $\Delta$CE is the
  performance penalty from losing the Slow-module.}
\label{tab:hlrole_ablation}
\scalebox{0.88}{%
\begin{tabular}{lccc}
\toprule
Model & CE (normal) & CE (H frozen) & $\Delta$CE \\
\midrule
HRM NT=4  (N=2,T=2) & 4.375 & 5.281 & +0.906 \\
HRM NT=8  (N=2,T=4) & 4.375 & 5.938 & +1.563 \\
HRM NT=12 (N=4,T=3) & \textbf{4.313} & \textbf{8.313} & \textbf{+4.000} \\
\bottomrule
\end{tabular}%
}
\end{table}

The Slow-module contributes significantly across all configurations, and
the contribution scales strongly with $M$: disabling the Slow-module
for NT=12 causes a 4.0-nat CE increase, nearly doubling the loss.
This directly shows that the performance advantage of HRM is not
attributable to weight sharing alone---the two-speed \emph{hierarchy}
is doing essential computational work that a flat iterative structure
cannot replicate.

\textbf{(B) Gate specialization on real text.}
Table~\ref{tab:hlrole} shows step-level statistics for HRM NT=12
on real val text.
Gate values on real text ($0.28$--$0.33$) differ substantially
from the near-uniform $0.50$ observed on zero-input tokens,
confirming that the gate responds to linguistic content.
Gate mean increases monotonically across steps, indicating the
Fast-module progressively retains more of its accumulated state as the
representation stabilizes.
The step immediately following an H-firing shows a slight gate
increase (0.306 vs 0.300 averaged over other steps), consistent
with the Fast-module consolidating the newly injected H-state before
accepting new fast-path updates.

\textbf{(C) $z_H$/$z_L$ alignment trajectory.}
The cosine similarity between $z_H$ and $z_L$ increases from 0.048
at step~0 to 0.390 at step~11 for NT=12.
This progressive alignment indicates that Slow-module is actively tracking
the evolution of $z_L$ across cycles: rather than maintaining an
independent representation, the slow path converges toward the
fast-path's refined state, suggesting a progressive aggregation
dynamic rather than a static context vector.

\begin{table}[h]
\centering\footnotesize\setlength{\tabcolsep}{2pt}
\caption{HLRole(B--C) step statistics for HRM NT=12 (N=4, T=3)
  on real val text. $\checkmark$: H fires.
  Gate values 0.28--0.33 (vs.\ 0.50 on zero tokens) confirm
  content-sensitive gating.}
\label{tab:hlrole}
\scalebox{0.72}{%
\begin{tabular}{ccrrrrr}
\toprule
Step & H? & gate & $\|z_L\|$ & $\|z_H\|_{\text{bef}}$ & $\|z_H\|_{\text{aft}}$ & $\cos(z_H,z_L)$ \\
\midrule
0  &   & 0.281 & 16.36 & 56.57 & 56.57 & 0.048 \\
1  &   & 0.268 & 10.32 & 56.57 & 56.57 & 0.040 \\
2  & $\checkmark$ & 0.281 &  8.73 & 56.57 & 10.21 & 0.084 \\
3  &   & 0.281 &  7.94 & 10.21 & 10.21 & 0.085 \\
4  &   & 0.288 &  7.76 & 10.21 & 10.21 & 0.080 \\
5  & $\checkmark$ & 0.292 &  7.68 & 10.21 &  6.85 & 0.134 \\
6  &   & 0.308 &  7.70 &  6.85 &  6.85 & 0.174 \\
7  &   & 0.313 &  7.79 &  6.85 &  6.85 & 0.182 \\
8  & $\checkmark$ & 0.316 &  7.82 &  6.85 &  8.11 & 0.296 \\
9  &   & 0.327 &  7.91 &  8.11 &  8.11 & 0.334 \\
10 &   & 0.330 &  7.97 &  8.11 &  8.11 & 0.342 \\
11 & $\checkmark$ & 0.331 &  7.99 &  8.11 &  8.61 & 0.390 \\
\bottomrule
\end{tabular}%
}
\end{table}

Taken together, the three analyses converge on the same conclusion:
the Slow-module is not a passive redundancy but an active component
whose contribution scales with the depth of iteration.
The ablation (A) quantifies this contribution directly in terms of
language modeling quality; the gate and norm trajectories (B--C)
show the mechanism through which the hierarchy operates on real text.

The NT=4 configuration (shortest cycle count) shows the smallest
ablation penalty (+0.91~nats), consistent with the slow path
having less time to integrate context before the output is read out.

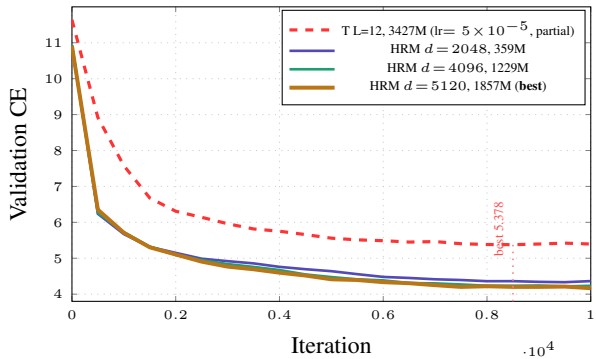
\begin{figure}[t]
\centering
\begin{tikzpicture}
\begin{axis}[
  width=\columnwidth, height=5.5cm,
  xlabel={\small Iteration}, ylabel={\small Validation CE},
  xmin=0, xmax=10000, ymin=3.8, ymax=12.0,
  xtick={0,2000,4000,6000,8000,10000},
  ytick={4,5,6,7,8,9,10,11},
  xticklabel style={font=\tiny},
  yticklabel style={font=\tiny},
  xlabel style={font=\small},
  ylabel style={font=\small},
  legend style={font=\tiny, at={(0.99,0.99)}, anchor=north east,
                row sep=-2pt, inner sep=2pt},
  clip=true,
  grid=major, grid style={dotted, gray!40},
  tick style={major tick length=2pt},
]
\addplot[color=red!80, line width=1.2pt, dashed] coordinates {
  (0,11.639)(500,8.917)(1000,7.578)(1500,6.678)(2000,6.309)
  (2500,6.140)(3000,5.961)(3500,5.816)(4000,5.752)(4500,5.659)
  (5000,5.559)(5500,5.512)(6000,5.489)(6500,5.450)(7000,5.464)
  (7500,5.404)(8000,5.383)(8500,5.378)(9000,5.396)(9500,5.421)(10000,5.396)};
\addlegendentry{T L=12, 3427M (lr$=5\!\times\!10^{-5}$, partial)}
\addplot[color={rgb,255:red,83;green,74;blue,183}, line width=1pt] coordinates {
  (0,10.921)(500,6.237)(1000,5.673)(1500,5.316)(2000,5.150)
  (2500,4.987)(3000,4.919)(3500,4.857)(4000,4.760)(4500,4.695)
  (5000,4.640)(5500,4.557)(6000,4.480)(6500,4.451)(7000,4.414)
  (7500,4.394)(8000,4.361)(8500,4.362)(9000,4.343)(9500,4.333)(10000,4.364)};
\addlegendentry{HRM $d\!=\!2048$, 359M}
\addplot[color={rgb,255:red,29;green,158;blue,117}, line width=1pt] coordinates {
  (0,10.940)(500,6.264)(1000,5.723)(1500,5.325)(2000,5.097)
  (2500,4.949)(3000,4.838)(3500,4.758)(4000,4.672)(4500,4.549)
  (5000,4.476)(5500,4.413)(6000,4.384)(6500,4.308)(7000,4.301)
  (7500,4.270)(8000,4.235)(8500,4.234)(9000,4.241)(9500,4.217)(10000,4.233)};
\addlegendentry{HRM $d\!=\!4096$, 1229M}
\addplot[color={rgb,255:red,186;green,117;blue,23}, line width=1.5pt] coordinates {
  (0,10.910)(500,6.353)(1000,5.711)(1500,5.304)(2000,5.104)
  (2500,4.906)(3000,4.767)(3500,4.693)(4000,4.598)(4500,4.519)
  (5000,4.414)(5500,4.394)(6000,4.334)(6500,4.305)(7000,4.250)
  (7500,4.202)(8000,4.220)(8500,4.202)(9000,4.202)(9500,4.211)(10000,4.165)};
\addlegendentry{HRM $d\!=\!5120$, 1857M (\textbf{best})}
\draw[red!50, dotted, line width=0.6pt]
  (axis cs:8500,3.8) -- (axis cs:8500,5.378)
  node[anchor=south, font=\tiny, text=red!70, rotate=90, xshift=6pt] {best 5.378};
\end{axis}
\end{tikzpicture}
\caption{EqualFLOPs learning curves (10k steps, corrected LR).
Transformer L=12 (3427M, dashed red) now converges with the corrected
LR ($5\!\times\!10^{-5}$, \texttt{max\_norm}$=0.3$), reaching val CE 5.396
at 10k steps, but plateaus after iter 8500 (best 5.378).
Initial grad spike at iter 500 (norm$=109$) delays convergence,
leaving a gap of 1.03--1.23 nats versus HRM variants at 10k.
HRM variants converge steadily to 4.16--4.36 (4--10$\times$ fewer parameters).}
\label{fig:equalflops}
\end{figure}

\subsection{Weight Matrix Memory Analysis}
\label{sec:weight_memory}

We first state the weight matrices that constitute a single
\texttt{CausalAttnBlock}, then count how many independent copies each
architecture must store.

\paragraph{CausalAttnBlock weight matrices.}
A single block ($d$: model dim, $d_{ff}=4d$: FFN dim) contains:
\begin{equation}
\begin{split}
\mathcal{W} \;=\;\;
&\underbrace{W_{qkv}}_{\scriptstyle d\times 3d}
\,,\;
\underbrace{W_o}_{\scriptstyle d\times d}
\,,\;
\underbrace{W_1,\,W_2}_{\scriptstyle d\times d_{ff}}
\\[3pt]
&\,,\;
\underbrace{W_3}_{\scriptstyle d_{ff}\times d}
\,,\;
\underbrace{\gamma_1,\,\gamma_2}_{\scriptstyle \mathbb{R}^d}
\end{split}
\label{eq:block_weights}
\end{equation}
with total parameter count:
\begin{equation}
|\mathcal{W}| \;=\; 3d^2 + d^2 + 2 \cdot 4d^2 + 4d^2 + 2d
             \;=\; 14d^2 + 2d.
\label{eq:block_size}
\end{equation}
For $d=4096$: $|\mathcal{W}| = 268{,}443{,}648 \approx 268\,\text{M}$
parameters, requiring $\approx 537\,\text{MB}$ in bfloat16.

\paragraph{Transformer: $L$ independent copies.}
A Transformer with $L$ layers stores $L$ independent instantiations of
$\mathcal{W}$, one per layer:
\begin{equation}
\begin{split}
\Theta_{\mathrm{TF}} &\;=\;
\bigl\{\mathcal{W}^{(1)},\,\mathcal{W}^{(2)},\,\ldots,\,\mathcal{W}^{(L)}\bigr\},
\\
|\Theta_{\mathrm{TF}}| &\;=\; L\,(14d^2+2d) \;=\; O(Ld^2).
\end{split}
\label{eq:tf_params}
\end{equation}
At $d=4096$, $L=12$: $|\Theta_{\mathrm{TF}}| \approx 3.22\,\text{B}$
parameters ($\approx 6.4\,\text{GB}$ bf16), all distinct.

\paragraph{HRM-LM: 3 shared instances.}
HRM-LM contains exactly three \texttt{CausalAttnBlock} instantiations
with independent parameters:
$\mathcal{W}_I$ (input enc.), $\mathcal{W}_L$ (Fast-module), and $\mathcal{W}_H$ (Slow-module).
$\mathcal{W}_L$ is reused $N \times T$ times per forward pass;
$\mathcal{W}_H$ is reused $N$ times.
The total stored weight is:
\begin{equation}
|\Theta_{\mathrm{HRM}}| \;=\; 3 \cdot (14d^2 + 2d) \;+\; P_{\mathrm{rest}}
\;\approx\; 3|\mathcal{W}| + P_{\mathrm{rest}},
\label{eq:hrm_attn_params}
\end{equation}
where $P_{\mathrm{rest}}$ covers the embedding, gate, and output
projection ($\approx 424\,\text{M}$ at $d=4096$).
The attention block portion is $3|\mathcal{W}| \approx 805\,\text{M}$
regardless of $N$, $T$, or $S$.

\paragraph{Savings summary.}
\begin{equation}
\frac{|\Theta_{\mathrm{TF}}^{\mathrm{attn}}|}{|\Theta_{\mathrm{HRM}}^{\mathrm{attn}}|}
\;=\; \frac{L}{3},
\qquad
\text{e.g.,} \\
\tfrac{L{=}12}{3} = 4{\times}, \quad
\tfrac{L{=}4}{3} \approx 1.3{\times}.
\label{eq:savings_ratio}
\end{equation}
Table~\ref{tab:param_savings} lists absolute values at $d=4096$.

\begin{table}[h]
\centering\scriptsize
\setlength{\tabcolsep}{3pt}
\caption{Attn block weight memory ($d\!=\!4096$, bf16).
  $^\dagger$Times $\mathcal{W}_L$ executes per forward pass.}
\label{tab:param_savings}
\resizebox{\columnwidth}{!}{%
\begin{tabular}{lrrrrc}
\toprule
Model & Blocks & Params & Mem & Steps$^\dagger$ & Saving\\
\midrule
TF L=4   & 4 indep. & 1.07B & 2.1\,GB & 4              & ---\\
TF L=12  & 12 indep. & 3.22B & 6.4\,GB & 12             & ---\\
HRM      & 3 shared  & 0.81B & 1.6\,GB & $N\!\times\!T$ & $4\times$\\
\bottomrule
\end{tabular}%
}
\end{table}

\noindent\textbf{Key point.}
The $4\times$ savings against Transformer L=12 come purely from
replacing 12 independent $\mathcal{W}^{(l)}$ with 3 shared
$\{\mathcal{W}_I, \mathcal{W}_L, \mathcal{W}_H\}$.
The number of recurrence steps $N \times T$ does \emph{not} affect
stored parameter count---only the number of distinct weight sets does.
This is the unconditional memory saving, valid regardless of sequence
length, batch size, or deployment platform.

\subsection{Parameter Count Derivation}

For $d=4096$, $H=16$, $V=50{,}257$, the total stored parameter count
is derived as follows.
Let $P_{\mathrm{cab}} = 14d^2 + 2d$ denote the shared
\texttt{CausalAttnBlock} size (used three times: \texttt{input\_proj},
\texttt{L\_net.attn}, \texttt{H\_net.attn}).
\begin{align*}
P_{\mathrm{emb}} &= Vd = 205{,}852{,}672 \\
P_{\mathrm{inp}} &= P_{\mathrm{cab}} = 268{,}443{,}648 \\
P_L &= 3d + 6d^2 + P_{\mathrm{cab}} + d^2 + 1 = 385{,}896{,}449 \\
P_H &= 2d + 4d^2 + P_{\mathrm{cab}} + d^2 + 1 = 352{,}337{,}921 \\
P_{\mathrm{out}} &= 3d + 9d + d^2 + 1 = 16{,}826{,}369 \\
\hline
P_{\mathrm{total}} &= \mathbf{1{,}229{,}357{,}059}
\end{align*}
The output head is weight-tied to the embedding ($W_{\mathrm{emb}}^\top$),
so it adds zero additional parameters.

\subsection{Compute and Memory Trade-off}
\label{sec:tradeoff}

Table~\ref{tab:tradeoff} quantifies the empirically observed trade-off,
now including the measured per-token forward-pass latency from Latency
(Section~\ref{subsec:latency}).

\begin{table}[h]
\centering\scriptsize\setlength{\tabcolsep}{3pt}
\caption{Measured trade-offs ($d=4096$, seq=1024, bfloat16).
  Latency: no-cache forward pass, single RTX 6000.}
\label{tab:tradeoff}
\setlength{\tabcolsep}{3pt}
\resizebox{\columnwidth}{!}{%
\begin{tabular}{lccc}
\toprule
Metric & TF L=4 & HRM $N\!\times\!T\!=\!8$ & HRM $N\!\times\!T\!=\!12$ \\
\midrule
Stored params    & 1280M  & 1229M   & 1229M \\
KV cache (theory)& $O(4nd)$  & $O(8nd)$  & $O(12nd)$ \\
Latency (seq=1024)& 10.5\,ms & 31.1\,ms & 46.8\,ms \\
Latency ratio    & $1\times$ & $2.97\times$ & $4.47\times$ \\
Peak GPU (seq=1024)& 2698\,MB & 2656\,MB & 2656\,MB \\
val CE (10k)     & 4.562  & \textbf{4.239} & 4.233 \\
\bottomrule
\end{tabular}%
}
\end{table}

\noindent\textbf{Memory analysis: two distinct savings.}

\textit{1. Parameter memory (primary, always valid).}
Because all $M$ recurrent steps share a single weight set $\theta$,
stored parameters shrink from $O(Ld^2)$ to $O(d^2)$.
For the equal-FLOPs baseline (Transformer L=12, $d=4096$, bf16):
$12 \times 4096^2 \times 2\,\text{bytes} \approx 4.4$\,GB saved in
model weights alone.
This reduction is unconditional---independent of sequence length,
batch size, or $M$.
For our equal-parameter experiment, both models have $\sim$1.23--1.28B
stored parameters by construction; the saving appears at the Transformer
L=12 level (Table~\ref{tab:equalflops}).

\textit{2. KV cache memory (secondary, conditional on $L > M$).}
Inference KV cache scales as $O(Mnd)$ for HRM-LM vs $O(Lnd)$ for an
$L$-layer Transformer, a ratio of $M/L$.
In our equal-parameter setting ($L{=}4$, $M{=}12$), HRM-LM uses
$3\times$ \emph{more} KV cache---KV cache saving is absent here.
The saving becomes meaningful when $L \gg M$, e.g.\ $L{=}32$--$96$
paired with $M{=}12$ gives ratio $0.1$--$0.4$.

Recent KV cache optimizations such as FlashAttention-3~\cite{shah2024flashattention3}
and PagedAttention~\cite{kwon2023pagedattention} operate orthogonally to
HRM-LM's architectural memory reduction: FlashAttention-3 reduces the
\emph{working memory} required during attention computation through
IO-aware tiling, while PagedAttention reduces \emph{KV cache fragmentation}
in serving systems.
Both can be applied on top of HRM-LM's shared parameter structure, as
each recurrent step is itself a standard Transformer forward pass.
In principle, the combination yields HRM-LM's $L$-fold weight reduction,
FlashAttention-3's reduced attention memory, and PagedAttention's
efficient cache management simultaneously---a stack inaccessible with
any single technique alone.

\noindent\textbf{Computational cost (measured).}
The sequential data dependency $z_L^{(s,i)} \leftarrow f(z_L^{(s,i-1)})$
prevents compiler-level parallelism across steps.
Latency measures a $\mathbf{4.5\times}$ latency increase for
HRM $N\!\times\!T\!=\!12$ versus Transformer L=4 at seq=1024
(46.8\,ms vs.\ 10.5\,ms, single RTX 6000).
Latency scales linearly with $M$: NT=4 is $2.0\times$, NT=8 is $3.0\times$,
NT=12 is $4.5\times$, with per-step cost stabilizing at $\approx$3.9\,ms
for $M \geq 8$.
The fundamental trade-off is:
\begin{center}
\fbox{\parbox{0.88\columnwidth}{\centering\small
  \textbf{HRM-LM trades compute time for total inference memory.}\\[2pt]
  \textit{Always}: stored weights shrink $L$-fold ($O(Ld^2)\to O(d^2)$).\\
  \textit{When $L > M$}: KV cache also shrinks by factor $M/L$.\\
  Cost: $\approx 2\text{--}3\times$ longer iteration time.}}
\end{center}

\section{Discussion and Limitations}
\label{sec:discussion}

\subsection{Main Findings}
\label{subsec:findings}

This paper investigates whether \emph{iterative refinement with a shared
two-speed transformation} can serve as a viable alternative to
\emph{stacking independent layers} for contextual representation in
language modeling.
The experimental evidence is cautiously positive, with significant caveats
that follow each finding.

\paragraph{Finding 1: Equal-parameter quality advantage.}
HRM-LM outperforms Transformer L=4 by 0.28--0.32 nats at equal stored
parameters ($\sim$1.23B, 10k steps, EqualParam).
HPSearch (budget-matched LR/warmup grid, 8 configurations per model) recovers
the advantage under matched tuning: HRM 4.308 vs.\ T-L4 4.543 ($+0.235$ nats).

\textit{Caveat---architecture vs.\ optimization effect not fully separated.}
MultiSeed shows that under \emph{identical} hyperparameters (no per-model
tuning), T-L4 (4.697) outperforms HRM NT=12 (4.896)---the reverse of the
EqualParam result.
EqualParam recovers HRM's advantage only after per-model LR, warmup,
initialization, and clipping are individually tuned.
The paper's narrative---repeated cycles of failure, correction, and
re-experiment---makes it difficult for a reader to determine what portion of
the final gap is a \emph{structural} property of HRM and what portion reflects
that HRM received more hyperparameter attention.
A clean separation would require either (a) a theoretical argument that HRM's
optimal initialization ($\sigma=0.02/\sqrt{M}$) and clipping
(\texttt{max\_norm}$=K/M$) are \emph{predictable} from the architecture rather
than empirically discovered, or (b) a joint grid search of equal budget applied
simultaneously to both architectures.
HPSearch approximates (b) for HRM vs.\ T-L4; the T-L12 comparison remains
subject to the LR-search asymmetry described in Finding 2.
The strongest defensible claim is: \emph{``Under matched parameter count and
matched hyperparameter search budget, HRM shows a consistent structural
advantage.''}

\paragraph{Finding 2: Equal-FLOPs comparison with T-L12 (TF-Baseline).}
T-L12 (3427M), extended to 60,000 steps with the best-found LR
($3\!\times\!10^{-5}$) and confirmed as a plateau ($\Delta$CE $<$0.001 nats
over the final 3,500 steps), reaches val CE 4.250---still 0.073 nats above
HRM NT=12 at 10k (4.177), despite $2.8\times$ more stored parameters and
$6\times$ more training steps.

\textit{Caveat---LR search asymmetry.}
HRM results aggregate across many LR configurations (EqualParam, LongTrain,
Ablation, HPSearch), while T-L12 was compared against only two LR settings
($5{\times}10^{-5}$ and $3{\times}10^{-5}$).
HPSearch was conducted for T-L4 but not T-L12.
A reader cannot rule out that a broader LR/warmup grid search for T-L12
would reduce its val CE, narrowing or eliminating the 0.073-nat gap.
This asymmetry is a genuine limitation of the TF-Baseline comparison.

\paragraph{Finding 3: Two-speed hierarchy vs.\ flat iteration (UniTF).}
The Universal Transformer (UniTF) comparison isolates the structural
contribution of the Fast/Slow hierarchy most cleanly.
UniTF at 1218M parameters (matched to HRM's 1229M) plateaus at $\approx$7.6
nats across five runs with varied LR, warmup, and initialization, while HRM
converges to 4.177---a gap of 3.4 nats that cannot be attributed to parameter
count or hyperparameter settings.
This is the paper's strongest piece of architectural evidence.

\textit{Caveat---single scale and dataset.}
All results are from $\sim$1.2B parameters on OpenWebText, 10k--60k steps.
Whether the advantage persists at 3B, 7B, or 70B is entirely open.
From a scaling-law perspective, HRM's compute-optimal frontier may diverge
from Transformer's; no such analysis is provided.

The most important open question is whether HRM's advantage survives at
production scale ($d\!\sim\!10{,}000$, $L\!\sim\!100$, as in GPT-3-class
models).
Three concrete mechanisms may erode it:

\begin{enumerate}[leftmargin=1.5em, itemsep=1pt, topsep=2pt]

  \item \textbf{Equal-parameter comparison becomes structurally degenerate.}
        At $M\!=\!100$ shared steps, the equal-parameter Transformer baseline
        would need only $L\!\approx\!1$---far too shallow to be meaningful.
        The relevant comparison shifts from ``equal stored parameters'' to
        ``equal FLOPs,'' which changes the axis of the claim entirely.

  \item \textbf{Shared weights may hit a representational ceiling.}
        A depth-100 Transformer can learn 100 qualitatively distinct
        transformations arranged hierarchically (syntax $\to$ semantics
        $\to$ pragmatics).
        HRM with $M\!=\!100$ iterations of a \emph{shared} map is more
        constrained: repeated application of the same transformation
        tends toward a fixed point, potentially limiting representational
        diversity regardless of $M$.
        At $M\!=\!12$ this is not empirically limiting; at $M\!\gg\!12$
        the question is open.

  \item \textbf{Gradient coverage degrades with depth.}
        At $M\!=\!12$, $K\!=\!8$ gives K/M$\,{=}\,67\%$---well above the
        $\approx$15--20\% stability threshold.
        Scaling to $M\!=\!100$ with the same $K$ gives K/M$\,{=}\,8\%$,
        below the threshold.
        Full BPTT ($K\!=\!M\!=\!100$) is memory-infeasible.
        No remedy (chunked BPTT, second-order methods) has been tested in HRM.

\end{enumerate}

These mechanisms suggest that scaling to $d\!\sim\!10{,}000$, $L\!\sim\!100$
is not a straightforward extrapolation, and that empirical validation at
intermediate scales ($M\!=\!48$, $M\!=\!96$) is a prerequisite before
stronger claims about production-scale applicability can be made.

\paragraph{Finding 4: $K/M$ gradient coverage as training principle.}
The Ablation identifies $K$ (TBPTT window) as the most impactful
hyperparameter: $K\!=\!8$ (K/M$=67\%$) achieves val CE 4.052, the best
single result in the ablation suite.
Keeping K/M above $\approx$15--20\% is empirically associated with stable
convergence.
The S-ablation reveals that $S\!>\!1$ supervision passes hurt due to
implicit $S{\times}$ LR amplification; correcting LR by $1/S$ is the fix.

\textit{Caveats---$K$ and $N{\times}T$ incompletely characterized.}
$K\!=\!M$ (full BPTT) was not run; the best $K$ is therefore unconfirmed.
The $T$ ablation reveals a non-monotone pattern ($T\!=\!1$ best,
$T\!=\!2{\approx}T\!=\!3$ worst, $T\!=\!6$ partially recovers) that the
current K/M-coverage explanation does not predict.
Proposition~2 (gradient amplification) is consistent with the observation
that larger $K$ yields lower CE, but verifying it would require showing
gradient norm or effective-rank trajectories per $K$---currently absent.
These remain design intuitions rather than established principles.

\subsection{Memory--Latency Trade-off}

Parameter sharing unconditionally reduces stored weights from $O(Ld^2)$
to $O(d^2)$ and---when $L > M$---also reduces KV cache by factor $M/L$.
The cost is $\approx\!2\text{--}5\times$ longer per-token latency (GenSpeed:
HRM NT=12 generates at 61.3 tok/sec vs 318.6 for T-L4).
HRM is best positioned as a \textbf{memory-efficient alternative} in
memory-bandwidth-limited environments (edge, CPU-only, fixed-budget deployments),
not as a latency-superior replacement.

\subsection{Applicability to Physical AI and Embodied Systems}
\label{subsec:physical_ai}

The memory and structural properties of HRM-LM suggest particular
applicability to \emph{physical AI} systems---robots, autonomous vehicles,
and other embodied agents---where on-device inference is subject to tight
memory, power, and latency constraints simultaneously.

\paragraph{Memory footprint on edge hardware.}
Onboard computers for humanoid robots and autonomous vehicles
(e.g.\ NVIDIA Jetson AGX Orin, automotive SoCs) lack the HBM found
in data-center GPUs; available memory bandwidth is typically
$\sim\!100$\,GB/s (vs.\ $\sim\!3{,}350$\,GB/s for H100 HBM3E).
At this bandwidth tier, the $L$-fold reduction in stored weights
translates directly into wall-clock latency reduction.
HRM NT=12 ($\approx$2.5\,GB, $d=4096$) occupies roughly 4\% of
Jetson AGX Orin's 64\,GB LPDDR5 memory,
whereas T-L12 ($\approx$13.7\,GB) occupies 21\%---leaving substantially
more headroom for perception, planning, and sensor buffers.
The compact HRM NT=36 variant (359M parameters, 744\,MB) fits entirely
within 16\,GB systems and is a candidate for latency-tolerant onboard
language understanding without GPU HBM.

\paragraph{Structural alignment with hierarchical control.}
Robotic control systems are inherently hierarchical in time:
low-level controllers operate at 50--200\,Hz (joint torques, reflexive
responses), while high-level planners operate at 1--10\,Hz (task
decomposition, situational awareness).
HRM's Fast/Slow two-speed structure maps naturally onto this separation:
the Fast-module (firing every step) corresponds to rapid local refinement,
while the Slow-module (firing every $T$ steps) corresponds to slower
global context integration.
Setting $T$ to match the frequency ratio between control levels offers
a structural inductive bias that flat architectures lack.
This alignment is a motivating hypothesis for future work on HRM-based
Vision-Language-Action (VLA) models; the present paper does not
validate it empirically.

\paragraph{KV cache reduction for continuous observation streams.}
Embodied agents process continuous streams of visual and language
observations, requiring long effective context windows.
At sequence length $n\!=\!4096$, T-L32 KV cache requires
$32\!\times\!4096\!\times\!d\!\times\!2\,\text{bytes} \approx 1.1$\,GB
(at $d\!=\!4096$, bfloat16), growing linearly with both depth and
context length.
HRM with $M\!=\!12$ and $L\!=\!32$ (i.e.\ $L > M$) reduces this to
$12\!\times\!4096\!\times\!d\!\times\!2 \approx 0.4$\,GB---a $63\%$
reduction that directly enables longer effective context on memory-limited
onboard hardware.

\paragraph{Weight reuse and cache warming.}
Because HRM reuses the same two weight matrices ($\theta_L$, $\theta_H$)
for all $M$ recurrent steps, these tensors remain resident in CPU or
GPU L3/L2 cache between steps.
By contrast, a Transformer with $L$ independent layers must load
distinct parameter tensors for each layer transition.
On CPU inference (e.g.\ llama.cpp-style deployment on consumer hardware),
this cache reuse provides effective bandwidth
$\approx$200--500\,GB/s (L3 hit rate) rather than
$\approx$100\,GB/s (DDR5 DRAM bandwidth), partially offsetting the
$2\text{--}5\times$ sequential latency overhead of recurrence.
Quantitative validation of this cache-warming effect in robotics
inference pipelines is left as future work.

\paragraph{Limitations in physical AI contexts.}
The $2\text{--}5\times$ slower token generation of HRM remains the
primary barrier for hard real-time applications requiring sub-millisecond
response.
HRM is better suited for \emph{soft real-time} tasks
(natural language command parsing, scene description, action narration)
where latency budgets of 50--500\,ms are acceptable,
and where the memory and power savings from reduced weight footprint
outweigh the throughput cost.
Whether HRM's structural advantages translate to better performance on
standard VLA benchmarks (e.g.\ OpenX-Embodiment, LIBERO, RoboAgent)
is an empirical question not addressed in this paper.

\subsection{Comparison with Vector Quantization Frameworks}

HRM-LM and post-hoc KV-cache quantization methods such as
TurboQuant~\cite{zandieh2025turboquant} target orthogonal memory bottlenecks.

\begin{table}[h]
\centering\scriptsize
\caption{HRM-LM vs.\ TurboQuant~\cite{zandieh2025turboquant}: orthogonal targets.}
\label{tab:comparison}
\setlength{\tabcolsep}{3pt}
\resizebox{\columnwidth}{!}{%
\begin{tabular}{lcc}
\toprule
Property & HRM-LM & TurboQuant \\
\midrule
Target             & Stored model weights          & KV cache activations \\
Stored param.\ mem.& $O(d^2)$, $L$-fold $\checkmark$ & Unchanged \\
KV cache memory    & $M/L$ reduction ($L{>}M$ only) & ${\geq}4.5\times$ $\checkmark$ \\
Quality cost       & None (CE improved)            & Quality-neutral at 3.5\,bits \\
Requires retraining& Yes                           & No \\
Latency overhead   & $2\text{--}5\times$ slower    & Negligible \\
\bottomrule
\end{tabular}%
}
\end{table}

HRM-LM reduces stored weights unconditionally; TurboQuant compresses
the KV cache at inference time without touching weights.
The two can be combined: HRM-LM first reduces the weight footprint,
then TurboQuant further compresses the resulting model's KV cache---
a combined memory profile inaccessible to either approach alone.

To illustrate the potential scale of savings, consider a hypothetical
deployment of HRM-LM with $M=12$, $d=4096$ serving long sequences ($n=4096$).
The following estimates are derived from each method's reported compression
ratios applied independently; the combination has not been empirically validated:
\begin{itemize}[leftmargin=*, itemsep=1pt, topsep=2pt]
  \item \textbf{HRM-LM alone}: stored weights $\approx 1.2$\,GB (vs 6.4\,GB for T-L12),
        KV cache $\approx 0.75$\,GB (12 steps $\times$ 2 layers $\times$ bfloat16).
  \item \textbf{+ TurboQuant at 3.5\,bits}: KV cache would be further compressed
        $\approx 4.5\times$ to $\approx 0.17$\,GB, based on
        TurboQuant's reported quality-neutral rate on LongBench-E~\cite{bai2023longbench}.
  \item \textbf{+ QJL residual quantization~\cite{zandieh2024qjl}}:
        inner-product estimation with 1-bit residuals could enable sub-2-bit
        KV cache at further quality trade-off.
\end{itemize}
If these savings compose without interference, the combined system---HRM-LM
weight compression + TurboQuant/QJL KV compression---opens up the possibility
of ultra-lightweight long-context inference at severely memory-constrained
deployments, though empirical validation of this combination remains future work.
The serial inference bottleneck of HRM-LM ($5.2\times$ slower than T-L4)
remains the primary deployment limitation; combining with speculative
decoding~\cite{leviathan2023fast} is a separate avenue for future work.

\textbf{Deployment guidance.}
Memory-constrained / quality-first (edge, fixed-budget): HRM-LM preferred.
Long-context serving with KV-cache bottleneck: TurboQuant or QJL preferred.
Combined (weight + KV compression): HRM-LM + TurboQuant is the most
promising direction for ultra-lightweight long-context inference, pending empirical validation.

\subsection{Comparison with SSM-based Models}
\label{subsec:ssm_comparison}

We considered including Mamba~\cite{gu2023mamba} and RWKV~\cite{peng2023rwkv}
as additional baselines.
Before explaining why we did not, a factual clarification is necessary.

\paragraph{Clarification: Mamba and RWKV are memory-efficient, not memory-hungry.}
RWKV uses linear attention with exponential decay and carries no KV cache at
inference time; inference memory is $O(d)$---constant in both context length
and layer count.
Mamba's selective state-space model likewise maintains a fixed-size state per
layer ($O(d \cdot d_{\text{state}})$, with $d_{\text{state}}\!\approx\!16$--$64$),
independent of sequence length.
Both models are among the most memory-efficient architectures at inference time
available in 2024--2025, outperforming standard Transformers on long-context
memory benchmarks by a wide margin.
Earlier versions of this manuscript incorrectly described these models as
memory-intensive; that claim is retracted.
This clarification also sharpens the positioning of HRM-LM: its memory advantage
is specifically in \emph{stored weight count} ($O(d^2)$ vs.\ $O(Ld^2)$),
not in inference-time KV cache, where Mamba and RWKV already dominate.

\paragraph{Why they were not included: design philosophy mismatch.}
HRM-LM's core claim concerns the \emph{stacking vs.\ iteration} axis:
whether iterating a \emph{shared-weight} transformation can replace
stacking independent parameterizations.
Mamba and RWKV address a different question---whether Transformer
\emph{attention} can be replaced by sub-quadratic recurrences---without
sharing weights across depth.
Neither reduces stored parameter count relative to depth.
A direct performance comparison would therefore conflate three orthogonal
axes: (1) weight sharing vs.\ independent weights, (2) attention vs.\
recurrence, and (3) quadratic vs.\ linear context scaling.
The Universal Transformer (UniTF) is the correct ablation target because
it holds axes (2) and (3) fixed, isolating only the weight-sharing
structure that HRM-LM claims to improve upon.

A comparison with Mamba or RWKV would be meaningful for a different
research question: ``given the same inference-time memory budget, which
architecture achieves better quality?''
That comparison is out of scope for the present paper but is an important
direction for future work, particularly for long-context ($n\!\geq\!4096$)
settings where HRM-LM's KV cache savings (conditional on $L > M$) compete
directly with Mamba and RWKV's $O(1)$ inference state.

\paragraph{Why training was infeasible: kernel availability.}
Both models require custom CUDA kernels for parallel scan
(\texttt{mamba-ssm}~\cite{gu2023mamba}) to achieve practical training speed.
In our pure-PyTorch implementation, a 1.2B Mamba model required
approximately 67 seconds per training step (vs.\ 3--5 seconds for HRM
or Transformer), making 10,000-step experiments infeasible.
A chunked-scan RWKV variant reduced this to ${\approx}$4.5 seconds,
but the design-philosophy mismatch remained the primary reason for exclusion.

\paragraph{Broader context: why efficient recurrences have not replaced Transformers.}
The fact that RWKV and Mamba have not supplanted Transformers as of 2026---despite
solving the quadratic attention and KV-cache bottlenecks---is itself informative
and directly relevant to the scope of this paper.
Four factors appear to be responsible, each highlighting a dimension along which
``solving efficiency'' is insufficient.

\textit{(1) Scaling laws and general capability.}
Transformers follow Chinchilla-optimal scaling~\cite{kaplan2020scaling}
remarkably well across reasoning, coding, and multimodal tasks.
At frontier scale (70B--405B parameters), Transformer-based models (GPT-4o,
Claude 3.5, Gemini 2.0) still match or exceed Mamba/RWKV on broad
benchmarks such as MMLU, GSM8K, and HumanEval.
Efficient recurrences excel on long-context language modeling and
memory-constrained settings, but their general-capability scaling
at frontier scale remains less established.

\textit{(2) The Tensor Core paradox.}
Modern accelerators (H100, Blackwell, TPU v5) are optimized for the
matrix multiplications underlying self-attention.
Mamba/RWKV's SSM and recurrent computations underutilize Tensor Cores,
so that training throughput---even after Mamba-2's parallel scan improvements
---remains behind Transformer + FlashAttention-3 + vLLM stacks.
The ecosystem (Hugging Face, vLLM, quantization tooling) is also
Transformer-centric, imposing an adoption barrier.

\textit{(3) Training stability at scale.}
Transformer training is fully parallelizable across the sequence dimension.
Recurrent architectures, even with parallel scans, impose sequential
dependencies that limit large-scale pre-training throughput.
At trillion-token pre-training scale, Transformers converge more stably;
Mamba/RWKV at 100B+ parameter scale remain less proven.

\textit{(4) Hybrid architectures as the emerging paradigm.}
The 2025--2026 SOTA is neither pure Transformer nor pure SSM/linear-attention.
Models such as Jamba, Nemotron-3, and Falcon-Mamba interleave Transformer
and Mamba/RWKV layers, achieving efficiency gains without sacrificing
the general-capability advantages of attention.
This suggests the field is converging on complementarity rather than replacement:
SSM/linear-attention layers as efficient building blocks that augment, not
eliminate, attention.

These observations sharpen the scope of the present paper.
HRM-LM investigates a question orthogonal to the Mamba/RWKV project:
not ``can attention be replaced by sub-quadratic recurrence?''
but ``given that we keep attention, can \emph{weight sharing with internal
hierarchy} substitute for layer stacking?''
The fact that attention replacement remains incomplete after three years of
intensive work on Mamba/RWKV suggests that the Transformer's attention
mechanism carries inductive biases that are non-trivial to replicate.
Whether HRM-LM's two-speed hierarchy provides sufficient bias for those
tasks where Mamba/RWKV currently fall short of Transformer performance
is an important open question, but one that is beyond the scope of
experiments at $\sim$1.2B parameters presented here.

\subsection{Limitations}
\label{subsec:limitations}

The structural caveats on the main findings are discussed inline in
Section~\ref{subsec:findings}.
The following missing experiments remain for future work:
\begin{enumerate}[leftmargin=*, itemsep=1pt, topsep=2pt]
  \item \textit{Stability boundary}: minimum $K/M$ for stable training
        ($\approx$15--20\% in current experiments) not systematically
        characterized for $T > 2$.
  \item \textit{Full BPTT}: $K\!=\!M$ (gradient window covers all $M$ steps)
        not run; the performance ceiling of $K$ is unconfirmed.
  \item \textit{Longer contexts}: KV cache advantage is most pronounced at
        $n \geq 2048$, but all results use $n=1024$.
  \item \textit{Multi-scale and multi-dataset}: single scale ($\sim$1.2B),
        single dataset (OpenWebText), no downstream task evaluation.
  \item \textit{T-L12 LR grid search}: HPSearch was not run for T-L12;
        a broader grid search would resolve the LR-asymmetry concern in
        the TF-Baseline comparison.
\end{enumerate}

\subsection{Path to Large-Scale Validation}
\label{subsec:scaling_path}

The question of whether HRM-LM's advantage survives at production scale
($d\!\sim\!10{,}000$, $L\!\sim\!100$, $M\!\sim\!100$) is the central open
problem, and answering it requires first solving the \emph{gradient coverage
bottleneck} identified in Section~\ref{subsec:findings}.

At $M\!=\!12$, $K\!=\!8$ gives K/M$\,{=}\,67\%$---stable.
Scaling to $M\!=\!100$ with fixed $K$ gives K/M$\,{\approx}\,8\%$---below
the empirical stability threshold.
Full BPTT ($K\!=\!M\!=\!100$) is memory-infeasible at 7B+ scale.
Three directions may resolve this:

\begin{enumerate}[leftmargin=1.5em, itemsep=2pt, topsep=2pt]

  \item \textbf{Chunked BPTT.}
        Partition the $M$ steps into chunks of size $C$; backpropagate
        through each chunk independently and accumulate gradients.
        This maintains K/M$\,{=}\,C/M$ without the memory cost of full BPTT,
        at the expense of weaker long-range gradient signal.
        Feasible immediately; needs empirical validation at
        $M\!=\!48$--$96$.

  \item \textbf{Hierarchical BPTT.}
        Use separate gradient windows for Fast and Slow modules:
        $K_{\text{fast}}$ covers recent Fast steps (local signal),
        $K_{\text{slow}}$ covers all Slow-module firing events (global signal).
        This is the most natural fit for HRM's two-speed structure---the
        Slow module fires every $T$ steps, so $K_{\text{slow}}\!=\!N$
        Slow events corresponds to a $K_{\text{fast}}\!=\!M$ effective
        window without storing the full gradient graph.
        This direction is unexplored and may be the key enabler for
        large-scale HRM training.

  \item \textbf{Second-order methods.}
        K-FAC~\cite{martens2015optimizing} or Shampoo-style curvature
        estimation can partially compensate for weak gradient coverage
        by improving the effective update direction.
        Computationally expensive at 70B+ scale but potentially
        applicable at 7B as a proof of concept.

\end{enumerate}

\subsubsection{M=48 Scaling Results (10k steps)}
\label{subsubsec:m48}

We ran a scaling experiment at $M\!=\!48$ (4$\times$ the primary $M\!=\!12$),
keeping $d\!=\!4096$ and stored parameters fixed at $\approx\!1.23$B,
for 10,000 steps.
Three models were evaluated:

\begin{table}[h]
\centering\footnotesize\setlength{\tabcolsep}{3pt}
\caption{M=48 results (10k steps, $d=4096$, $\approx$1.23B params).
  M=12 reference values (10k, converged) shown for comparison.}
\label{tab:m48_results}
\resizebox{\columnwidth}{!}{%
\begin{tabular}{lrrrrrrr}
\toprule
Model & M & K/M & 3k & 5k & 7k & 10k & grad$_{10k}$ \\
\midrule
HRM $N\!=\!8,T\!=\!6$   & 48 & 33\% & 5.054 & 4.443 & 4.255 & \textbf{4.172} & 1.281 \\
HRM $N\!=\!12,T\!=\!4$  & 48 & 25\% & 4.819 & 4.741 & 4.509 & 4.519 & 1.570 \\
UniTF flat$^\dagger$    & 48 & 33\% & 5.735 & 6.461 & 6.823 & 6.101 & 0.812 \\
\midrule
HRM $N\!=\!4,T\!=\!3$ (M=12) & 12 & 67\% & --- & --- & --- & \textit{4.177} & --- \\
UniTF (M=12)                 & 12 & 17\% & --- & --- & --- & \textit{7.578} & --- \\
\bottomrule
\end{tabular}}
{\small
$^\dagger$UniTF M=48 suffered initialization overflow (iter 0 val CE$\approx$79.6 vs.\ $\approx$10.9 for HRM)
and exhibits severe oscillation (3k:5.74$\to$7k:6.82$\to$10k:6.10); its 10k result is not reliable
and should not be compared directly to UniTF M=12.}
\end{table}

\textbf{Finding 1: HRM M=48 matches HRM M=12.}
The most striking result is that HRM $N\!=\!8, T\!=\!6$ at $M\!=\!48$ achieves
val CE $\mathbf{4.172}$ at 10k steps---essentially identical to HRM M=12
(4.177, $\Delta = 0.005$ nats).
This means that scaling the iteration count from $M\!=\!12$ to $M\!=\!48$,
while reducing K/M from 67\% to 33\%, preserves quality.
The Fast/Slow hierarchy remains effective at 4$\times$ the iteration depth.

\textbf{Finding 2: N=8, T=6 outperforms N=12, T=4 at M=48.}
Despite N=12,T=4 converging faster in early steps (3k: 4.819 vs.\ 5.054),
N=8,T=6 overtakes it and achieves 4.172 vs.\ 4.519 at 10k.
The higher gradient coverage (K/M$=33\%$ vs.\ 25\%) and longer Slow-module
period ($T\!=\!6$ vs.\ $T\!=\!4$) appear more favorable at this scale.

\textbf{Finding 3: UniTF M=48 does not replicate the M=12 plateau cleanly.}
UniTF M=48 reaches 6.101 at 10k, but shows severe oscillation
(3k:5.74$\to$7k:6.82$\to$10k:6.10) attributable to initialization
overflow at iter 0 (val CE $\approx$79.6, indicating the shared-block
initialization formula $\text{std}=0.02/\sqrt{2Md/d_{\text{base}}}$
is insufficient at $M\!=\!48$).
The M=12 plateau at $\approx$7.6 nats was a clean, monotone convergence;
UniTF M=48's unstable trajectory makes a direct comparison unreliable.
A re-run with corrected initialization is needed before concluding whether
the plateau phenomenon is reproduced at $M\!=\!48$.

\textit{Cautious interpretation.}
The hierarchy gap at $M\!=\!48$ is nominally
$6.101 - 4.172 = 1.929$ nats (vs.\ $3.401$ nats at M=12), but this
comparison is confounded by the UniTF initialization failure.
The robust conclusion is that HRM's quality is stable across
$M\!=\!12$ to $M\!=\!48$ (4.177$\to$4.172), suggesting the Fast/Slow
hierarchy does not degrade with increased iteration depth at these scales.
Whether flat iteration (UniTF) would plateau above or below HRM under
fair initialization at $M\!=\!48$ remains to be confirmed.

Empirical validation should next re-run UniTF M=48 with corrected
initialization (std$=0.02/\sqrt{4Md/d_{\text{base}}}$, implemented in v4),
then proceed to $M\!=\!96$ before 7B-scale experiments.
Until then, the current paper's contribution is best understood as
\emph{evidence that the Fast/Slow hierarchy maintains quality at
$M\!=\!48$; the flat-iteration comparison at this depth requires
a methodologically clean re-run}.

\section{Conclusion}

This paper investigates whether hierarchically structured shared-weight
recurrence can match independent-layer stacking in language modeling.
The central empirical finding is a clean contrast between flat and
hierarchical shared-weight iteration: the Universal Transformer (UniTF,
same block repeated uniformly, 1218M) plateaus at $\approx$7.6 nats across
five independent runs, while two-speed HRM (Fast/Slow hierarchy, 1229M)
converges to 4.177---a 3.4-nat gap at equal parameter count that is robust
to hyperparameter variation.
This observation supports the hypothesis that \emph{internal structure
within the shared iteration}---not the iteration count or parameter count---
is what makes shared-weight recurrence viable.

The paper's other comparisons carry important caveats that should be read
carefully.
Under standardized equal-hyperparameter evaluation (MultiSeed),
HRM NT=12 (4.896) trails Transformer L=4 (4.697) by 0.20 nats
under equal hyperparameters---but at $2\times$ lower stored-weight
memory ($\approx$2.5\,GB vs.\ $\approx$5.1\,GB).
Whether this trade-off is acceptable depends on deployment context;
the quality gap closes under matched tuning budget (HPSearch).
The budget-matched HPSearch recovers a 0.23-nat margin for HRM over T-L4
(4.31 vs.\ 4.54), but relies on a single seed and was not extended to T-L12.
The TF-Baseline comparison (T-L12 at 60k steps, val CE 4.250) shows HRM
trailing by 0.073 nats at 10k---though this comparison is subject to LR
search asymmetry.
These results are best interpreted as: \emph{HRM shows a structural
advantage under matched parameter count and matched tuning budget;
the advantage is not unconditional.}

Parameter sharing yields an unconditional $L$-fold reduction in stored
weight memory ($O(Ld^2) \to O(d^2)$) and---when $L > M$---a reduction
in KV cache by factor $M/L$.
These savings come at a cost of $\approx 2$--$5\times$ slower per-token
generation due to sequential recurrence, and in the equal-parameter
experiments ($L=4$, $M \in \{8,12\}$) the KV cache saving is absent
since $M > L$.
The memory benefit is therefore architecture-dependent and deployment-context-specific.

All results are at $\sim$1.2B parameters on OpenWebText.
Generalization to larger scales ($d\!\sim\!10{,}000$, $L\!\sim\!100$),
other datasets, downstream tasks, and long-context settings is unverified.
The gradient coverage bottleneck (K/M degradation with $M$), the
representational ceiling hypothesis, and whether the equal-parameter
comparison axis remains meaningful at scale are the three risks that must
be addressed before stronger claims can be made.
This paper reports what was found at 1.2B scale; it does not claim to
have answered the question at larger scales.

\medskip
\noindent\textbf{Code availability.}
Training scripts, model implementations, and experiment configurations
will be released on GitHub upon publication.

\appendix

\section{Theoretical Observations: Full Statements and Proofs}
\label{app:theory}

\subsection{Stability Bound}

\begin{proposition}[Conditional Stability]
\label{prop:stability}
Suppose that at step $i$, the attention-block output satisfies
$\|h^{(i)}W_{\mathrm{out}}\|_\infty \leq C$ for some constant $C < \infty$.
Then:
\begin{equation}
  \|z_L^{(s,i)}\|_\infty \leq
  \max\!\Bigl(\|z_L^{(s,i-1)}\|_\infty,\;
              \alpha C\Bigr).
\end{equation}
Consequently, $\|z_L^{(s,i)}\|_\infty \leq \max\!\bigl(\|z_L^0\|_\infty,\;
\alpha\max_{j\leq i} C_j\bigr)$ for all $i \leq M$.
\end{proposition}

\begin{proof}
Write the update element-wise for index $k$:
$[z_L^{(s,i)}]_k = g_k [z_L^{(s,i-1)}]_k + (1-g_k)\alpha [h^{(i)}W_{\mathrm{out}}]_k$.
Since $g_k \in (0,1)$, this is a convex combination.
By the triangle inequality,
$|[z_L^{(s,i)}]_k| \leq \max\!\bigl(\|z_L^{(s,i-1)}\|_\infty,\;\alpha C\bigr)$.
Taking $\sup_k$ and applying induction completes the proof.
\end{proof}

\subsection{Gradient Amplification Argument}

\begin{proposition}[Gradient Amplification under Path Alignment]
\label{thm:amplification}
Let $p_i$ denote the gradient contribution of step $i$ through the shared
parameter $\theta$.
Assume: (i) $\|\prod_{j=i+1}^M J_j\| \geq 1$ for all $i$ (non-collapse);
(ii) $\cos\angle(p_i, p_M) \geq 0$ for all $i \in \{M-K+1,\ldots,M\}$
(path alignment).
Then $\|\nabla_\theta L_K\| \geq K \cdot \|\nabla_\theta L_1\|$.
\end{proposition}

\begin{proof}
Under (i), $\|p_i\| \geq \|p_M\| = \|\nabla_\theta L_1\|$ for all $i$.
Under (ii), $\|\nabla_\theta L_K\| = \|\sum_i p_i\| \geq \sum_i \|p_i\|
\geq K\|\nabla_\theta L_1\|$.
\end{proof}

\section{UniTF: Initialization Details and Per-run Logs}
\label{app:unitf}

\subsection{Initialization overflow and corrected formula}

The original UniTF initialization $\sigma\!=\!0.02/\sqrt{M}$ is derived for
$d\!=\!4096$ and does not account for width-dependent variance accumulation.
Attention output variance scales as $d \cdot \sigma^2$; for larger $d$ the
per-step contribution grows proportionally.
The width-adjusted formula that restores the same per-step scale as the
$d\!=\!4096$ baseline is:
\[
\sigma = \frac{0.02}{\sqrt{M \cdot d / d_\text{base}}}, \quad d_\text{base} = 4096.
\]
For $d\!=\!5760$, $M\!=\!12$: $\sigma \approx 0.00487$ (15\% smaller than
the failing $0.00577$).
A further $2\times$ safety factor was applied in later runs:
$\sigma = 0.02/\sqrt{2Md/d_\text{base}} \approx 0.00344$.

\subsection{Per-run results for UniTF $d\!=\!5760$}

Table~\ref{tab:unitf_full} reports all $d\!=\!5760$ runs.
Init overflow at iter 0 in v2b/v2c is a seed-dependent bfloat16 numerical
artifact; all models recover to normal loss by iter 500 and plateau at the
same $\approx$7.6 level regardless of the iter-0 CE.

\begin{table}[h]
\centering\scriptsize\setlength{\tabcolsep}{4pt}
\caption{UniTF --- all UniTF $d\!=\!5760$ runs (820M, $M\!=\!12$).
  All three converge to the same $\approx$7.6 plateau.}
\label{tab:unitf_full}
\begin{tabular}{llll}
\toprule
Run & lr / warmup & iter 0 CE & val CE plateau \\
\midrule
v2  & $5\!\times\!10^{-5}$ / 2000 & 11.98 (normal) & 7.624 @ 10k \\
v2b & $10^{-4}$ / 500             & 91.7 (bad seed) & 7.613 @ 2k \\
v2c & $10^{-4}$ / 500, v3-init   & 109.9 (bad seed) & 7.646 @ 4.5k \\
\bottomrule
\end{tabular}
\end{table}

\subsection{UniTF $d\!=\!7296$ run log}

Table~\ref{tab:unitf_7296} shows the val CE trajectory for the
parameter-matched $d\!=\!7296$ run (1218M, lr$=4\!\times\!10^{-5}$, warmup$=2000$).

\begin{table}[h]
\centering\scriptsize\setlength{\tabcolsep}{5pt}
\caption{UniTF $d\!=\!7296$ val CE trajectory (1218M, 10k complete).}
\label{tab:unitf_7296}
\begin{tabular}{rr}
\toprule
iter & val CE \\
\midrule
0    & 139.04 \\
500  & 7.909 \\
1000 & 7.651 \\
1500 & 7.627 \\
2000 & 7.612 \\
2500 & 7.553 \\
3000 & 7.599 \\
3500 & 7.567 \\
4000 & 7.548 \\
4500 & 7.596 \\
5000 & 7.592 \\
5500 & 7.593 \\
9000 & 7.600 \\
9500 & 7.578 \\
10000 & \textbf{7.578} \\
\bottomrule
\end{tabular}
\end{table}

\subsection{Why does UniTF plateau at $\approx$7.6? Gradient norm analysis}

The $\approx$7.6 ceiling is reproducible across widths and hyperparameters,
suggesting a structural rather than optimization-driven limitation.
The following observations support this interpretation.

\textbf{Gradient norm behavior.}
In the $d\!=\!7296$ run, grad norm is stable throughout: values of
2.03--6.41 at checkpoints 500--5500, with a single mild spike at iter 2005
(grad norm$=10.06$, above the clip threshold of $1/6 \cdot M = 0.167$,
triggered by the lr-warmup dynamics and promptly recovered).
This indicates the optimizer is healthy and the plateau is not caused by
gradient vanishing, explosion, or numerical instability.

\textbf{Flat-iteration ceiling interpretation.}
UniTF ($M=12$) applies the same weight matrix at every step, lacking
the Slow-module's slow integration path.
Without the hierarchical slow path, the model converges to a fixed point
of the recurrence $z \leftarrow f_\theta(z, x)$ that is shallower than
what the Fast-module alone can represent.
The ceiling at $\approx$7.6 is consistent across $d\!=\!4096$ (6.4, already converged
to a different ceiling), $d\!=\!5760$ (7.6 in 3 runs), and $d\!=\!7296$ (7.6 at
5.5k steps), implying the limiting factor is the flat-iteration architecture
rather than capacity.

\textbf{Cross-run consistency.}
Across five UniTF runs at $d\geq5760$ with varied lr ($5\!\times\!10^{-5}$--$10^{-4}$),
warmup (500--2000), and initialization, val CE at 2k--5.5k ranges
$7.548$--$7.651$ (std $\approx$0.03).
This robustness to optimization hyperparameters strengthens the interpretation
that 7.6 is an architectural ceiling, not a local minimum of the loss landscape.

\bibliographystyle{unsrtnat}

\end{document}